\documentclass[a4paper]{amsart} 

\usepackage{amsmath,amsthm,amssymb,amsfonts,mathrsfs,color,hyperref, mathtools,crop, graphicx, enumitem, todonotes}
\usepackage[]{geometry}

\theoremstyle{plain}
\begingroup
\newtheorem{theorem}{Theorem}[section]
\newtheorem*{theorem*}{Theorem}
\newtheorem*{"theorem"}{``Theorem''}
\newtheorem{corollary}[theorem]{Corollary}

\newtheorem{lemma}[theorem]{Lemma}
\endgroup

\theoremstyle{definition}
\begingroup

\endgroup

\theoremstyle{remark}
\begingroup
\newtheorem{remark}[theorem]{Remark}
\newtheorem{example}[theorem]{Example}
\endgroup 

\numberwithin{equation}{section}
\newenvironment{pde}{\left\{\begin{array}{rll} } {\end{array}\right.}

\newcommand{\N}{\mathbb N}

\newcommand{\R}{\mathbb R} 
\newcommand{\E}{{\mathbb E}}
\renewcommand{\P}{{\mathbb P}}

\newcommand{\dist}{{\rm dist}}

\newcommand{\F}{{\mathcal F}}

\newcommand{\Risk}{\mathcal{R}}

\newcommand{\LRa} {\Leftrightarrow}
\newcommand{\Ra} {\Rightarrow}

\renewcommand{\d}{\mathrm{d}}

\newcommand{\ds}{\,\mathrm{d}s}

\renewcommand{\P}{\mathbb{P}}
\newcommand{\diag}{\mathrm{diag}}
\newcommand{\rk}{\mathrm{rk}}

\newcommand{\eps}{\varepsilon}
\newcommand{\average}{{\mathchoice {\kern1ex\vcenter{\hrule height.4pt
width 6pt depth0pt} \kern-9.7pt} {\kern1ex\vcenter{\hrule
height.4pt width 4.3pt depth0pt} \kern-7pt} {} {} }}

\allowdisplaybreaks

 \makeatletter
\@namedef{subjclassname@2020}{2020 Mathematics Subject Classification}
\makeatother
\newcommand\showlabel{\addtocounter{equation}{1}\tag{\theequation}}

\begin{document}

\title[SGD in machine learning: Global convergence]{Stochastic gradient descent with noise of machine learning type\\{\small Part I: Discrete time analysis}}


\author{Stephan Wojtowytsch}
\address{Stephan Wojtowytsch\\
Department of Mathematics\\
Texas A\&M University\\
155 Ireland Street\\
College Station, TX 77840
}
\email{stephan@tamu.edu}

\date{\today}


\begin{abstract}
Stochastic gradient descent (SGD) is one of the most popular algorithms in modern machine learning. The noise encountered in these applications is different from that in many theoretical analyses of stochastic gradient algorithms. In this article, we discuss some of the common properties of energy landscapes and stochastic noise encountered in machine learning problems, and how they affect SGD-based optimization. 

In particular, we show that the learning rate in SGD with machine learning noise can be chosen to be small, but uniformly positive for all times if the energy landscape resembles that of overparametrized deep learning problems. If the objective function satisfies a \L{}ojasiewicz inequality, SGD converges to the global minimum exponentially fast, and even for functions which may have local minima, we establish almost sure convergence to the global minimum at an exponential rate from any finite energy initialization. The assumptions that we make in this result concern the behavior where the objective function is either small or large and the nature of the gradient noise, but the energy landscape is fairly unconstrained on the domain where the objective function takes values in an intermediate regime.
\end{abstract}

\maketitle

\vspace{-12mm}

\setcounter{tocdepth}{1}
\tableofcontents

\section{Introduction}

Stochastic gradient descent algorithms play an important role in convex and non-convex optimization. They are used in machine learning when the computation of the exact gradient of the objective function is computationally costly, but stochastic approximations can be evaluated fairly cheaply. The stochastic noise is furthermore believed to aid the algorithm in non-convex optimization by allowing it to escape `bad' local minima and saddle points.

 this work, we analyze toy models for stochastic gradient descent in machine learning applications. These differ from more classical perspectives principally in the fact that the intensity of noise in estimating $\nabla f(\theta)$ depends on the value $f(\theta)$ of the objective function. By comparison, more classical works on SGD assume uniform $L^p$-bounds for $p\in [2,\infty)$, independently of $\theta$. Our toy models are inspired by overparametrized supervised learning, i.e.\ minimization problems where the objective function vanishes on a high-dimensional set. In this setting, SGD with ML noise has the following novel properties:

\begin{enumerate}
\item The learning rate has to be small in terms of the smoothness of the objective function and the noise intensity, but can be uniformly positive for all time.
\item If the learning rate remains uniformly positive, we can prove almost sure convergence to a global minimizer, not just a critical point that is not a strict saddle.
\end{enumerate}

Let us make these claims more precise. The crucial property of ML noise is that its variance scales at most linearly with the objective function $f$. In particular, if $f$ is small, so is the noise. If $f$ satisfies a \L{}ojasiewicz inequality (e.g.\ if $f$ is strongly convex), the learning rate has to be small enough for the noise to not exceed a certain strength, and compared to the Lipschitz constant of $\nabla f$. It does not have to vanish asymptotically to guarantee convergence to the global minimum. In this sense, SGD with ML noise resembles deterministic GD more than it does SGD with classical noise. For details, see Theorem \ref{theorem linear convergence}. 

If the learning rate is strictly positive, at any time there are two options:
\begin{itemize}
\item $f(\theta_t)$ is small and so is the noise. If $f$ has good properties on the set $\{f<\eps\}$ for some $\eps>0$, then with large probability we stay inside the set once we enter for the first time. Conditioned on this event, we expect to converge to a minimizer linearly.
\item $f(\theta_t)$ is large and so is the noise. If $f$ has nice properties on the set $\{f>S\}$ for some $S>0$, we expect SGD with ML noise to not let us escape to infinity despite the growing noise intensity. If additionally the set $\{f\leq S+1\}$ is not too spread out away from the set of minimizers and the noise is `uniformly unbounded' (but with bounded variance), then from any point in $\{f\leq S+1\}$ there is a positive probability of jumping into the set $\{f<\eps\}$ independently of the local gradient.
\end{itemize}

This intuition can be extended and made precise. In particular, under strong conditions on the target function and noise, we find that we converge linearly to a minimizer almost surely from any finite energy initialization. However, since the probability of jumping into the set $\{f<\eps\}$ may be exceedingly small, we note that we do not expect to observe this linear convergence on realistic time scales in complicated real world applications for poor initialization. A precise statement of our main result can be found in Theorem \ref{theorem global convergence}.

Stochastic gradient descent can be studied as a general tool in convex and non-convex optimization or specifically in the context of deep learning applications. In this article, we take a balanced approach by incorporating key features of the energy landscape and stochastic noise in deep learning problems, but without specializing to a specific machine learning model. 

The article is structured as follows. In Section \ref{section review}, we review known results on stochastic gradient descent and objective functions in deep learning. We deduce some properties of the energy landscape and stochastic noise which a suitable toy model should satisfy. Discrete time SGD with ML noise is discussed in Section \ref{section discrete}. Numerical examples in Section \ref{section numerical} illustrate the difference of SGD with classical and ML noise in a toy problem. Some proofs are postponed to the appendices. Continuous time results on continuous time SGD with noise of ML type are presented in a companion article \cite{continuous_sgd}.

The different parts of the article -- analysis of energy landscape and noise, discrete time analysis -- can be read independently, and a reader only interested in one chapter can easily skip ahead.

\subsection{Context} 

Stochastic gradient descent algorithms have been an active field of research since their inception in the seminal paper \cite{robbins1951stochastic}. Stochastic gradient descent (SGD) and advanced gradient descent-based optimization schemes with stochastic gradient estimates have been the subject of increased attention recently due to their relevance in neural network-based machine learning, see e.g.\ \cite{dieuleveut2017bridging,needell2014stochastic,pmlr-v89-vaswani19a,ward2019adagrad,xie2020linear, defossez2020convergence,allen2017natasha,moulines2011non, rakhlin2011making,jentzen2021strong,ghadimi2016mini, ghadimi2013stochastic,fehrman2020convergence,bottou2018optimization,bernstein2018convergence,bach2013non} and many more. A good literature review can be found in \cite{fehrman2020convergence}.

Theoretic guarantees for the convergence of SGD can be split into two categories:

\begin{itemize}
\item For convex target functions, SGD converges to the global minimum. 
\item For smooth non-convex target functions, SGD converges to a critical point, which is not a strict saddle.
\end{itemize} 

Both types of guarantee can be obtained under different conditions and either in expectation or almost surely, and both can be complemented with sharp rates. We review results of both types in Section \ref{subsection brief review}. Results may hold for the final iteration $\theta_t$ of SGD, the best parameter $\tilde \theta_t$ along the SGD trajectory up to time $t$, or a weighted average $\bar \theta_t$ of previous positions of SGD. We focus on guarantees for the final position since a list of previous values, or even a non-trivially weighted average of previous positions, is expensive to maintain in deep learning, where a model may have millions of parameters.

Recently, it has been noted in the SGD literature \cite{karimi2016linear} that convexity can be replaced by the condition that the objective function satisfies a \L{}ojasiewicz-inequality, i.e.
\begin{equation}\label{eq lojasiewicz first}
\Lambda \,f(\theta)\leq |\nabla f|^2(\theta) \qquad\forall\ \theta.
\end{equation}
This advance is particularly crucial since objective functions in deep learning typically have complicated geometries where the set of global minimizers is a submanifold of high dimension and co-dimension. Such functions are not usually convex, even in a neighbourhood of a minimizer.

The inequality \eqref{eq lojasiewicz first} holds in particular for all strongly convex functions, but may not hold for smooth convex functions, e.g.\ $f(x) = \sqrt{x^2+1}$. Like for convex functions, the only critical points of a function satisfying a \L{}ojasiewicz inequality lie in the set of global minimizers. The assumption of a global \L{}ojasiewicz inequality thus remains restrictive, and other models, e.g.\ functions satisfying local \L{}ojasiewicz inequalities at the critical level sets have been considered \cite{dereich2021convergence}. 

Classically, the only assumption on SGD noise is a uniform second moment bound. In this work, we consider more realistic noise model inspired by deep learning applications. In this setting, we can prove convergence to the global minimum for a more general class of functions which are neither required to be convex, nor to satisfy a \L{}ojasiewicz inequality. In this sense, we prove a first type guarantee under weaker second type assumptions. 

While we obtain fast (linear) convergence to the global minimum 
\[
f(\theta_t) - \inf f \leq Z \,\rho^t\qquad\text{rather than the rate}\qquad f(\theta_t) - \inf f \leq \frac{C}t,
\]
which is typical in the stochastic setting, we note that the random variable $Z$ is typically too large to yield meaningful bounds in practice, unless the target function and noise have particularly convenient properties. Nevertheless, the fast convergence allows us to circumvent a common `staying local' assumption, which is automatic for exponentially decaying sequences.

The guarantees we obtain hold almost surely (i.e.\ with probability $1$) over the choice of initial condition and stochastic gradient optimization. The importance of almost sure statements in this context has for example been emphasized in \cite{patel2020stopping} in the context of deriving stopping criteria which are `triggered' with probability $1$.

\subsection{Notation and conventions}

All random variables are defined on a probability space $(\Omega,\mathcal A, \P)$ which remains abstract and is characterized mostly as expressive enough to support a random initial condition and countably many iid copies of a random variable to select a gradient estimator. The dyadic product of two vectors $a, b$ is denoted by
\[
a\otimes b = a\cdot b^T, \qquad\text{i.e.} (a\otimes b)_{ij} = a_i b_j.
\]

\section{Energy landscapes and stochastic noise in machine learning}\label{section review}

\subsection{A brief review of stochastic gradient descent}\label{subsection brief review}

Stochastic gradient descent algorithms are a class of popular algorithms in machine learning to find minimizers of an objective function $f$. Instead of taking a small step in the direction $-\nabla f$ in every iteration, we choose the update direction randomly according to a random variable $g$ with expectation $-\nabla f$ and suitable $L^p$-bounds for some $p\geq 2$. 

More formally, we consider the following model.

\begin{enumerate}
\item $f:\R^m\to [0,\infty)$ is a $C^1$-function 
\item $\nabla f$ satisfies the one-sided Lipschitz-condition
\[
\langle \nabla f(\theta_1) - \nabla f(\theta_2), \theta_1-\theta_2\rangle \leq C_L\,|\theta_1-\theta_2|^2
\]
\item $(\Omega,\mathcal A, \P)$ is a probability space and $g:\R^m\times\Omega\to\R^m$ is a family of functions such that
\begin{enumerate}
\item $\E_{\xi\sim\P} g(\theta,\xi) = \nabla f(\theta)$ for all $\theta\in \R^m$ and
\item $\E_{\xi\sim\P} \big[ |g(\theta,\xi)- \nabla f(\theta)|^2\big] < \infty$ for all $\theta \in \R^m$.
\end{enumerate}
\end{enumerate}

The SGD algorithm associated to the family $g$ of gradient estimators is given by the time-stepping scheme
\begin{equation}
\theta_{t+1} = \theta_t - \eta_t\,g(\theta_t, \xi_t)
\end{equation}
where 
\begin{enumerate}
\item the initial condition $\theta_0$ is a random variable in $\R^m$,
\item $\eta_t>0$ is the {\em learning rate} (or time step size) in the $t$-th time step, and
\item $\{\xi_t\}_{t\geq 0}$ is a family of random variables in $\Omega$ with law $\P$, which are iid and independent of $\theta_0$.
\end{enumerate}

In theoretical works, we typically consider decaying learning rates. This is rooted in results such as this, originating from \cite{robbins1951stochastic}.

\begin{theorem}\label{theorem sgd 1}
Assume that $\nabla f$ is Lipschitz-continuous with constant $C_L>0$ and that 
\begin{itemize}
\item $f$ is convex or
\item $f$ satisfies the {\em \L{}ojasiewicz inequality} $\Lambda (f-\inf f) \leq |\nabla f|^2$ for some $\Lambda>0$.
\end{itemize}
If
\begin{equation}\label{eq classical noise}
\E_\xi\big[|g(\theta,\xi)-\nabla f(\theta)|^2\big]\leq \sigma^2 \qquad\forall\ \theta\in \R^m
\end{equation}
and the learning rates $\eta_t$ satisfy the {\em Robbins-Monro} conditions
\[
\sum_{t=0}^\infty \eta_t = \infty, \qquad \sum_{n=0}^\infty \eta_t^2 <\infty,
\]
then $\lim_{t\to\infty}\E\big[f(\theta_t)\big] = \inf f$. If $f$ satisfies the {\em \L{}ojasiewicz inequality} and $\eta_t = \frac2{\Lambda(t+1)}$, then
\[
\E\big[f(\theta_t) - \inf f\big] \leq \frac{2C_L\sigma^2}{\Lambda^2}\,\frac{\log(t+1)}t\,\E\big[ f(\theta_0)- \inf f\big]
\]
\end{theorem}

A proof for the \L{}ojasiewicz case (which in particular contains the uniformly convex case) can be found in \cite[Theorem 4]{karimi2016linear}, albeit under the assumption that $|\nabla f|$ is uniformly bounded, which contradicts the assumption that $f$ satisfies a \L{}ojasiewicz inequality (unless $f$ is constant). This can be weakened to the assumption that the noise has bounded variance \cite{bottou2018optimization}.

The first summation condition is equivalent to stating that we are solving a gradient-flow type equation on the entire real line. Violating it would introduce a finite time horizon, which would prevent convergence to the minimizer as $t\to \infty$ even in the non-stochastic setting and even for smooth and strongly convex objective functions. The second condition guarantees that the impact of gradient noise diminishes sufficiently as $t\to \infty$ such that the iterates $\{\theta_t\}_{t\geq 0}$ are not driven away from the global minimum in the long term by random oscillations. Unlike the first condition, it can be weakened \cite[Chapter 5]{kushner2003stochastic}. It has been suggested that the learning rate decay $\eta_ t= \bar \eta/t$ is optimal \cite[Section 4.1.2]{li2017stochastic}. 

If the objective function is non-convex and does not satisfy a \L{}ojasiewicz inequality, convergence to a critical point which is not a strict saddle can be guaranteed under suitable conditions on $f$ and $g$. We only give an imprecise statement.

\begin{theorem}\cite{mertikopoulos2020almost}\label{theorem sgd 2}
Under suitable conditions on the objective function $f$ and a condition of the type $\E\big[|g-\nabla f|^2\big] \leq \sigma^2$, the following holds: If $\eta_t = \frac{\eta_0}{(t+t_0)^p}$ for some $p\in (1/2,1]$, then
\begin{enumerate}
\item $f(\theta_t)$ converges to a random variable $Z$ almost surely and $Z$ lies in the set of critical values of $f$ almost surely.
\item $\sup_{t}|\theta_t|<\infty$ almost surely.
\item for every trajectory of SGD, there exists a connected component $\mathcal X$ of the set of critical points of $f$ such that $\lim_{t\to\infty} \dist(\theta_t, \mathcal X)=0$.
\item If $S$ is a connected component of the set of critical points of $f$ such that $D^2f(\theta)$ has a negative eigenvalue for all $\theta\in S$ (a ridge manifold), then $\P\big(\lim_{t\to\infty}\dist(\theta_t,S)=0\big) =1$. 
\end{enumerate}
\end{theorem}

In particular, an isolated strict saddle point is a zero-dimensional ridge manifold. Stronger statements, including a rate of convergence, are available at isolated local minima. The conditions on $f$ posed in \cite{mertikopoulos2020almost} are neither implied by our assumptions, nor are they more general. While sublevel sets of $f$ have to be compact in \cite{mertikopoulos2020almost} but not for us, we impose stronger conditions on the gradient away from the set of minimizers.


In practice, we typically choose learning rates according to a schedule of the form $\eta_t \equiv \eta^{(1)}$ for $0\leq t<T_1$, $\eta_t \equiv \eta^{(2)}$ for $T_1\leq t< T_2$, \dots with a discrete set of switching times $0< T_1<T_2 <\dots$ and learning rates $\eta^{(1)}> \eta^{(2)}>\dots$. Generally, the learning rates are chosen as large as possible without inducing numerical instability, and it is realistic to consider the setting where the learning rate is reduced only a finite number of times.

We consider the setting where even in the long term limit, the learning rate remains strictly positive. Under the conditions of Theorem \ref{theorem sgd 1}, we find that the energy quickly decays below a threshold depending on the learning rate.

\begin{theorem}\cite[Theorem 4]{karimi2016linear}\label{theorem sgd 3}
Assume that $f$ satisfies the {\em \L{}ojasiewicz inequality} $\Lambda (f-\inf f) \leq |\nabla f|^2$ and that $\nabla f$ is Lipschitz-continuous with Lipschitz-constant $C_L$. Assume further that the stochastic noise satisfies \eqref{eq classical noise}. If $\eta_t\equiv \eta < \frac{1}{C_L}$, then
\[
\E\big[f(\theta_t) - \inf f\big] \leq\left(1-\Lambda\eta\right)^t\E\big[f(\theta_0)-\inf f\big] + \frac{C_L\sigma^2}{2\Lambda}\,\eta\qquad\forall\ t\in\N.
\]
\end{theorem}

Theorems \ref{theorem sgd 1}, \ref{theorem sgd 2} and \ref{theorem sgd 3} made very weak assumptions on the nature of the stochastic noise in the estimators $g$.  Under different conditions, we can obtain better results. In particular, the type of noise encountered in the gradient estimators in machine learning is fairly special.

\subsection{On energy landscapes in machine learning}
Consider a parameterized function class
\[
h:\R^m \times \R^d \to \R^k
\]
where $\R^m$ is the space of parameters, $\R^d$ is the space in which the data is given, and $\R^k$ is a target space. In supervised learning problems, we seek to minimize functions of the form 
\[
L(\theta) = \frac1n \sum_{i=1}^n \ell\big( h(\theta,x_i), y_i\big)
\]
where
\begin{enumerate}
\item $\{(x_i,y_i)\}_{i=1}^n$ is our training set of data/output or data/label pairs in $\R^d\times \R^k$,
\item $\theta$ denotes the parameters of our function model, and
\item $\ell :\R^k\times \R^k\to [0,\infty)$ is a positive function.
\end{enumerate}
We use $\ell$ to establish a concept of similarity on $\R^k$ and optimize $\theta$ in order to make $h(\theta,x_i)$ ``similar'' to $y_i$ with respect to $\ell$, i.e.\ to make $\ell(h(\theta,x_i), y_i)$ small.

\begin{remark}
Depending on the context, both $L$ and $\ell$ are referred to as the loss function. At times, $L$ is referred to as `risk' and denoted by the symbol $\Risk$ instead.
\end{remark}

\begin{example}\label{example loss functions}
The prototypical example of a loss function in our context is MSE (mean squared error)/$\ell^2$-loss
\[\label{eq mse loss}
\ell(h,y) = \frac12\,\|h-y\|^2_{\R^k}.
\]
Another popular loss function in the setting of classification problems is softmax cross entropy loss
\[\showlabel\label{eq cross entropy}
\ell(h, y) = -\log\left(\frac{\exp(h\cdot y)}{\sum_{i=1}^k \exp(h\cdot e_i)}\right), \qquad y\in \{e_1,\dots, e_k\}.
\]
The energy landscapes in both types of learning problems are quite different since cross entropy loss is never zero (see \ref{example cross-entropy} for more details) and our results do not apply in the second setting.
\end{example}

In this article, we focus on $\ell^2$-type losses. Note that
\[
L(\theta) = 0 \qquad\LRa\qquad \frac1n\sum_{i=1}^n\big|h(\theta,x_i) - y_i\big|^2 = 0 \qquad\LRa \qquad h(\theta,x_i) = y_i \quad\forall\ 1\leq i\leq n.
\]
We give a modified version of a theorem of Cooper \cite[Theorem 2.1]{cooper2018loss}.

\begin{theorem}\cite{cooper2018loss}\label{theorem cooper}
Assume that 
\begin{enumerate}
\item the function model $h$ is overparametrized, i.e.\ $m>nk$,
\item the function model $h$ is so expressive that for any collection of outputs $y_1, \dots, y_n$ there exists $\theta\in \R^m$ such that 
\[
h(\theta, x_i) = y_i \qquad\forall\ 1\leq i \leq n,
\]
\item $h$ is $C^{m-nk+1}$-smooth in $\theta$.
\end{enumerate}
Then for Lebesgue-almost any choice of $(y_1,\dots, y_n)\in \R^{nk}$, the set
\[
N = \big\{\theta\in \R^m : L(\theta) = 0\big\}
\]
is a closed $m-nk$-dimensional $C^1$-submanifold of $\R^m$. If $h$ is Lipschitz-continuous in the first argument for every fixed $x$ and sufficiently overparameterized that it can fit arbitrary values at $n+1$ points, then $N$ is non-compact.
\end{theorem}

The conditions of this theorem are satisfied by certain overparametrized neural networks with smooth activation function \cite[Lemma 3.3]{cooper2018loss}, but also many other suitably expressive function classes. The condition $m\geq nk$ is implicit in the assumptions on expressivity and smoothness of $h$.

\begin{proof}
We can write $N = \Psi^{-1}(y)$ where
\[
\Psi:\R^m \to \R^{nk}, \qquad \Psi(\theta) = \begin{pmatrix} h(\theta, x_1)\\ \vdots\\ h(\theta,x_n)\end{pmatrix}, \qquad y = \begin{pmatrix} y_1\\ \vdots \\ y_n\end{pmatrix}.
\]
By assumption, $N$ is non-empty and if $y$ is a regular value of $\Psi$, then $N$ is a $C^1$-manifold of dimension $m-nk$. Since $\Psi$ is as smooth as $h$, we can conclude by Sard's theorem \cite{sard1942measure} that Lebesgue-almost every $y$ is a regular value of $\Psi$. As the pre-image of a single point under a continuous map, $N$ is closed.

If $h$ is sufficiently overparametrized that it can fit values not only at $n$ but at $n+1$ points, then for any $x_{n+1}\in\R^d$ and $y_{n+1}\in \R^k$ there exists $\theta\in \R^m$ such that $h(\theta, x_i) = y_{i}$ for $1\leq i\leq n+1$. In particular, for any choice $y_{n+1}\in\R^k$ we have $\theta \in N$. 

Let $y_{n+1}, y_{n+1}'\in\R^k$ and $\theta,\theta'$ associated indices in $N$. Since 
\[
\big|y_{n+1} - y_{n+1}'\big| = \big|h(\theta, x_{n+1}) - h(\theta', x_{n+1})\big| \leq C_h\,|\theta-\theta'|,
\]
we see $N$ cannot be compact since achieving very different outputs $y_{n+1}, y_{n+1}'$ requires very different parameters $\theta, \theta'$.
\end{proof}

We observe the following: If $L$ is convex, then the set of minimizers is convex. An $m-nk$-dimensional closed submanifold of $\R^m$ is a convex set if and only if it is an $m-nk$-dimensional affine subspace of $\R^m$. If the map $\theta\mapsto h(\theta,x)$ is generically non-linear, there is no reason to expect $N$ to be an affine space. Thus, we typically expect that $L$ is non-convex, even close to its set of minimizers. A rigorous statement on the non-convexity of objective functions in deep learning is given in Appendix \ref{appendix non-convexity}.

Thus, any toy model for the energy landscape $L:\R^m\to [0,\infty)$ of deep learning in overparametrized regression problems should have the following property: {\em There exists a non-compact closed $C^1$-manifold $N\subseteq \R^m$ such that $L(\theta)=0$ if and only if $\theta\in N$.} We can think of both the dimension and co-dimension of $N$ as being large.

\begin{remark}
While functions with these properties in general cannot be convex even in a neighbourhood of their set of minimizers, they may satisfy a \L{}ojasiewicz inequality of the form 
\[
\|\nabla f\|^2 \geq c\,f.
\]
One such function is $f(x,y) = \frac12\big(y-h(x)\big)^2$ for any $h\in C^1(\R)$ since
\[
\|\nabla f\|^2 \geq |\partial_yf|^2 = \big(y- h(x)\big)^2 = 2f.
\]
The assumption that a \L{}ojasiewicz inequality holds is standard in non-convex optimization, but may not be realistic in machine learning, where complex energy landscapes with many local minima, maxima, ridges and saddle points are observed. For a fascinating demonstration of how diverse these energy landscapes can be, see \cite{skorokhodov2019loss}. 
\end{remark}

However, we note that there are no critical points of high risk in square loss regression problems under very general conditions in the finite data or infinite data case.

\begin{lemma}\label{lemma no high energy critical points}
Let $\mu$ be a probability distribution on $\R^d\times \R^k$ with finite second moments (the data distribution). Assume that $h(\theta,\cdot)$ is a multi-layer perceptron with weights $\theta$ which maps $\R^d$ to $\R^k$. Define the risk functional
\[
\Risk(\theta) = \frac12 \int_{\R^d\times \R^k} \big|h(\theta,x)-y\big|^2\d\mu_{(x,y)}.
\]
Then there exists a constant $C>0$ such that
\[
|\theta|\cdot |\nabla \Risk(\theta)| \geq \Risk(\theta) - C.
\]
In particular, there exists $S>0$ such that 
\[
\Risk(\theta)>S \qquad \Ra\qquad \nabla \Risk(\theta)\neq 0.
\]
\end{lemma}

The function $\Risk$ is an empirical risk functional if $\mu$ is a finite sum of Dirac deltas and a population risk functional otherwise.

\begin{proof}
Let $h^*(x) = \E\big[y|x\big]$ and $\overline\mu$ the distribution of $x$ in $\R^d$ if $(x,y)$ is distributed according to $\mu$. Then
\begin{align*}
\Risk(\theta) &= \frac12 \int_{\R^d\times \R^k} \big|h(\theta,x)-h^*(x)\big|^2\d\mu_{(x,y)} + \frac12 \int_{\R^d\times \R^k} \big|h^*(x)-y\big|^2\d\mu_{(x,y)}\\
	&= \frac12 \int_{\R^d\times \R^k} \big|h(\theta,x)-h^*(x)\big|^2\d\overline\mu_{x} + C_\mu.
\end{align*}
For a neural network, we can decompose $\theta = (W, \hat \theta)$ and $h(\theta,x) = W\,\sigma\big(\hat h(\hat\theta,x)\big)$ where $\hat h:\R^d\to \R^l$ is a neural network with one layer less than $h$ and $W:\R^l\to \R^k$ is linear. If $\Risk(\theta)$ is larger than
\[
\int_{\R^d\times\R^k}|y|^2\,\d\mu_{(x,y)},
\]
then $h(\theta,x) \not\equiv 0$, so in particular $W\neq 0$. We can therefore compute
\begin{align*}
|\nabla_\theta \Risk| (\theta)&\geq \frac{|W\cdot \nabla_W \Risk|}{|W|}\\
	&= \frac1{|W|} \left|W\cdot\int_{\R^d} \big(h(\theta,x) - y\big)\cdot \nabla_Wh(\theta,x)\,\d\mu_{(x,y)}\right|^2\\
	&=\frac1{|W|} \left|\int_{\R^d} \big(h(\theta,x) - y\big)\left( W\nabla_Wh(\theta,x)\right)\,\d\mu_{(x,y)}\right|^2\\
	&= \frac1{|W|} \left|\int_{\R^d} \big(h(\theta,x) - h^*(x)\big) \cdot h(\theta,x)\,\d\overline\mu_{x}\right|
\end{align*}
since the representation of $h$ is linear in $W$, so $W\cdot \nabla_Wh(\theta,x) = h(\theta,x)$. Consequently
\begin{align*}
|W|\,|\nabla_\theta \Risk|(\theta)&\geq \int_{\R^d} \big|h(\theta,x) - h^*(x)\big|^2 - h^*(x)\cdot \big(h(\theta,x) - h^*(x)\big) \,\d\overline\mu_{x}\\
	&\geq \|h_\theta-h^*\|_{L^2(\P)}^2 - \|h_\theta-h^*\|_{L^2(\P)}\|h^*\|_{L^2(\mu)}\\
	&= 2\big(\Risk(\theta) - C_\mu\big) - \sqrt2\sqrt{\Risk(\theta) - C_\mu}\,\|h^*\|_{L^2(\mu)}.
\end{align*}
The claim follows since $2 \Risk(\theta) - \|h^*\|_{L^2(\mu)}\sqrt{2\,\Risk(\theta)} \geq \Risk(\theta) - \|h^*\|_{L^2(\mu)}^2$.
\end{proof}

\subsection{\L{}ojasiewicz landscapes}

We briefly discuss local \L{}ojasiewicz conditions in machine learning and beyond. For sufficiently `nice' objective functions, these can easily be seen to hold.

\begin{example}
Assume that $f:\R^m\to[0,\infty)$ is $C^2$-smooth and $N:= f^{-1}(0)$ is a $C^1$-manifold. Assume that
\begin{enumerate}
\item $D^2f$ is globally Lipschitz-continuous on the set $U = \{\theta : f(\theta) \leq 1\}$ and
\item there exists $0 < \lambda \leq \Lambda $ such that $\lambda |v|^2 \leq v^TD^2(\theta)v \leq \Lambda\,|v|^2$ for all $\theta \in N$ and $v\in (T_\theta N)^\bot$.
\end{enumerate}
Then there exists $\eps>0$ such that $f$ satisfies a \L{}ojasiewicz inequality on the set $U_\eps = \{\theta : f(\theta)<\eps\}$ since 
\begin{align*}
f(\theta + tv) &= \frac{t^2}2\,v^T D^2f(\theta) v + O(t^3)\\
\nabla f(\theta+tv) &= t\,D^2f(\theta)\,v + O(t^2)
\end{align*}
so if $v\bot T_\theta N$ and $|\theta|=1$ we have
\begin{align*}
f(\theta+tv) &= \frac{t^2}2\,v^T D^2f(\theta) v + O(t^3)\\
	& \leq \frac{\Lambda}2\,t^2 + O(t^3)\\
	& \leq \frac\Lambda{2\lambda^2} \left(\lambda^2\,t^2 + O(t^3)\right)\\
	&\leq \frac\Lambda{2\lambda^2} \left(\left|D^2f(\theta)\,v\right|^2\,t^2 + O(t^3)\right)\\
	&= \frac\Lambda{2\lambda^2}|\nabla f(\theta)|^2 + O(t^3).
\end{align*}
Since $f(\theta + tv) \geq \frac{\lambda}2t^2$ and due to the uniform Lipschitz bound, the smallness of the correction term is uniform over the set $\{f<\eps\}$.
\end{example}

\begin{example}[Quadratic regression]
For quadratic regression problems  
\[
L(\theta) = \frac1{2n}\sum_{i=1}^n \big(h(\theta,x_i)-y_i\big)^2
\]
we have
\begin{align*}
\nabla L(\theta) &= \frac1{n}\sum_{i=1}^n \big(h(\theta,x_i)-y_i\big)\,\nabla_\theta h(\theta,x_i)\\
D^2L(\theta) &= \frac1{n}\sum_{i=1}^n \nabla_\theta h(\theta,x_i) \otimes \nabla_\theta h(\theta,x_i) + \big(h(\theta,x_i)-y_i\big)\,D^2h(\theta,x_i).
\end{align*}
In particular if $L(\theta) =0$, the Hessian terms of $h$ drop out and
\[
D^2L(\theta) = \frac1{n}\sum_{i=1}^n \nabla_\theta h(\theta,x_i) \otimes \nabla_\theta h(\theta,x_i)
\]
The upper bound $\|D^2L\|\leq C$ is therefore satisfied on $N$ if the map $\theta\mapsto h(\theta,x_i)$ is Lipschitz-continuous. The lower bound corresponds to the assumption that the $n$ vectors $\nabla_\theta h(\theta,x_i)$ are bounded away from zero and bounded away from becoming linearly dependent. If this is the case, the lower bound
\[
\lambda|v|^2 \leq v^T D^2L(\theta)v
\]
holds since $v\in (T_\theta N)^\bot$ lies in the $n$-dimensional space spanned by $\{\nabla_\theta h(\theta,x_i) : 1\leq i\leq n\}$. These assumptions correspond to a function model in which we can always change the parameters by a small amount and change the model output by a positive amount, and in which the output at different data points $x_i, x_j$ is principally governed by different parameters. The Lipschitz-condition on $D^2L$ corresponds to the assumption that $h, \nabla_\theta h, D^2_\theta h$ are Lipschitz-continuous and bounded. 
\end{example}

\begin{example}[Cross-entropy classification]\label{example cross-entropy}
For classification problems, the cross-entropy loss function
\[
L(\theta) = -\E_{(x,i)\sim\mu} \left[ \log\left(\frac{\exp\big(h(x)\cdot e_i\big)}{\sum_{j=1}^k \exp\big(h(x)\cdot e_j\big)}\right)\right]
\]
never vanishes, so the energy landscape is fundamentally different from that of quadratic regression problems. We show that generally, a \L{}ojasiewicz condition of the type $\Lambda f \leq |\nabla f|^2$ cannot hold even for small values of the objective function.

The loss function
\[
\ell(h, i) = -\log\left(\frac{\exp (h\cdot e_i)}{\sum_{j=1}^k \exp(h\cdot e_j)}\right) = \log\left(\sum_{j=1}^k \exp(h\cdot e_j)\right) - h\cdot e_i
\]
satisfies
\begin{align*}
\nabla_h\ell(h,i) &= \sum_{l=1}^k \frac{\exp(h\cdot e_l)}{\sum_{j=1}^k\exp(h\cdot e_j)} \,e_l - e_i
\\
\big|\nabla_h\ell(h,i)\big|^2 &= \left|\frac{\exp(h\cdot e_i)}{\sum_{j=1}^k\exp(h\cdot e_j)}-1\right|^2 + \sum_{l\neq i} \left|\frac{\exp(h\cdot e_l)}{\sum_{j=1}^k\exp(h\cdot e_j)}\right|^2.
\end{align*}
With the counting density $p_l = \frac{\exp(h\cdot e_l)}{\sum_j\exp(h\cdot e_j)}$, we find that
\[
\ell = -\log(p_i) \geq 1-p_i, \qquad |\nabla \ell|^2 = (1-p_i)^2 + \sum_{l\neq i}p_l^2 \leq 2\,(1-p_i)^2,
\]
so 
\[
\frac12\big|\nabla_1\ell\big(h(\theta,x), i\big)\big|^2\leq (1-p_i)^2 \leq \min\big\{1,- \log(p_i)\big\}^2 \leq \min\big\{1,\ell\big(h(\theta,x), i\big)\big\}^2.
\]

Assume for the sake of contradiction that $\lambda\,L(\theta) \leq |\nabla L|^2(\theta)$ on the set $\{L<\eps\}$. Then the solution $\theta$ of the gradient flow equation $\dot\theta = -\nabla L(\theta)$ satisfies
\[
\frac{d}{dt} L(\theta(t)) = - |\nabla L|^2(\theta(t)) \leq -\lambda\,L(\theta(t))\qquad \Ra \qquad L(\theta(t)) \leq e^{-\lambda t}\,L(\theta(0)).
\]
In particular 
\begin{align*}
\big|\theta(t)-\theta(0)\big| &\leq \int_0^t \big|\nabla L(\theta(s))\big|\,\ds
	\leq 2\,\int_0^t L(\theta(s))\,\ds
	\leq 2\,L(\theta(0))\int_0^t e^{-\lambda s}\,\ds \leq \frac{2\,L(\theta(0))}\lambda.
\end{align*}
if $L(\theta(0))<\eps$. In particular, $\theta(t)$ is contained in a compact subset of $\R^m$ and $L(\theta(t))\not \to 0$ as $t\to\infty$, leading to a contradiction.
\end{example}

\subsection{Random selection SGD}

Consider a more general objective function $f:\R^m\to [0,\infty)$ of the form
\[
f(\theta) = \E_{\xi\sim\P} \phi(\theta,\xi)
\]
where $(\Omega,\mathcal A, \P)$ is a probability space and $\phi:\R^m\times \Omega\to [0,\infty)$ is a non-negative function such that
\begin{enumerate}[label=(G\arabic*), ref=(G\arabic*)]
\item for any fixed $\theta$, the function $\xi\mapsto \phi(\theta,\xi)$ is $\mathcal A$-measurable,
\item for any fixed $\xi$, the function $\theta\mapsto \phi(\theta,\xi)$ is continuously differentiable, and
\item for any compact set $K\subseteq\R^m$, there exists $\psi_K \in L^1(\P)$ such that 
\[
\phi(\theta,\xi) + \big|\nabla_\theta\phi(\theta,\xi)\big| \leq \psi_K(\xi)\qquad \forall\ \theta \in K
\]
$\P$-almost everywhere.
\end{enumerate}
The assumptions guarantee that expectation and parameter gradient commute \cite[Section 8.4]{konigsberger2013analysis}. A particular gradient estimator is
\[
g(\theta, \xi) = (\nabla_\theta\phi)(\theta,\xi)
\]
since
\[
\E_\xi g(\theta,\xi) = \E_\xi \big[ (\nabla_\theta\phi)(\theta,\xi)\big] = \nabla_\theta \big[\E_\xi \phi(\theta,\xi)\big] = \nabla f(\theta).
\]
Denote by $\Sigma$ the covariance matrix of the estimator $g$:
\[
\Sigma_{ij} (\theta) = \E_\xi \big[ \big(g_i(\theta,\xi) - \partial_i f(\theta)\big) \big(g_j(\theta,\xi)- \partial_jf(\theta)\big)\big],
\]
which we assume is well-defined:

\begin{enumerate}[label=(G\arabic*), ref=(G\arabic*)]\setcounter{enumi}{3}
\item for all $\theta\in \R^m$ we have $\nabla_\theta\phi(\theta,\cdot) \in L^2(\P)$.
\end{enumerate}

This model encompasses machine learning problems as discussed above with
\[
f(\theta) = \int_{\R^d\times \R^k} \ell\big(h(\theta,x),y\big)\,\d\mu_{(x,y)}, \qquad \xi = (x,y), \qquad g(\theta,\xi) = \partial_1\ell\big(h_\theta(x), y\big)\,\nabla_\theta h(\theta,x),
\]
if the loss function $\ell$ and the data distribution $\mu$ are compatible via growth conditions and moment bounds.

\subsection{SGD noise in machine learning}

We make an obvious observation.

\begin{lemma}\label{lemma no noise at minimum}
Assume there exists an $m-n$-dimensional manifold $N\subseteq \R^m$ such that $\phi(\theta,\xi) = 0$ for $\P$-almost all $\xi$. Then $\Sigma = 0$ everywhere on $N$ for the random selection gradient estimator.
\end{lemma}

\begin{proof}
Since $\phi(\theta,\xi) = 0$ an $\phi\geq 0$, we find that $\phi(\cdot,\xi)$ is minimal at $\theta$, so $\nabla_\theta\phi(\theta,\xi) =0$ for all $\xi$. Since there is no oscillation in the gradient estimator, the variance vanishes.
\end{proof}

\begin{remark}
A popular toy model for SGD is the family of gradient estimators
\[
g(x,\xi) = \nabla f(x) + \sigma \xi
\]
where $\sigma>0$ and $\xi$ is a standard Gaussian random variable. Unlike machine learning SGD, the estimators correspond to calculating exact gradients and deliberately perturbing them stochastically. Clearly, this toy model fails to capture key features of random selection SGD at the minimizing manifold.
\end{remark}

We prove a more quantitative version of Lemma \ref{lemma no noise at minimum}.

\begin{lemma}\label{lemma noise bounds}
Let $\mu$ be a data distribution and $h$ a parametrized function model such that $|\nabla_\theta h(\theta,x)| \leq C$ for $\mu$-almost every $x$. For a quadratic regression problem
\begin{align*}
L(\theta) &= \frac12 \int_{\R^d\times\R^k} \big|h(\theta,x) -y\big|^2\,\d\mu_{(x,y)} 
\end{align*}
the noise satisfies
\[
\E_{\xi\sim\P} \big[|g(\theta,\xi) - \nabla L(\theta)|^2\big] \leq \E_{\xi\sim\P} \big|g(\theta,\xi)\big|^2 \leq \|\nabla_\theta h \|_{L^\infty}^2\,L(\theta).
\]
For a cross-entropy classification problem
\begin{align*}
L(\theta) &= - \int_{\R^d\times\{1,\dots,k\}} \log\left(\frac{\exp\big(h(\theta,x)\cdot e_i\big)}{\sum_{j=1}^k \exp(h\big(\theta,x)\cdot e_j\big)}\right)\,\d\mu_{(x,i)}
\end{align*}
the noise satisfies
\[
\E_{\xi\sim\P} \big[|g(\theta,\xi) - \nabla L(\theta)|^2\big] \leq \E_{\xi\sim\P} \big|g(\theta,\xi)\big|^2 \leq 2\,\|\nabla_\theta h \|_{L^\infty}^2\,\min\big\{1,L(\theta)\big\}.
\]
\end{lemma}

\begin{proof}
{\bf Quadratic regression.}
\begin{align*}
\nabla L(\theta) &= \int_{\R^d\times\R^k} \big(h(\theta,x) - y\big)\cdot \nabla_\theta h(\theta,x)\,\d\mu_{(x,y)}\\
g(\theta,\xi) &= \big(h(\theta,x) - y\big)\cdot \nabla_\theta h(\theta,x)
\end{align*}
where $\xi = (x,y)$. Under the assumption that $\nabla_\theta h$ is uniformly bounded on $\R^m\times\R^d$, we observe that
\begin{align*}
\E_{\xi\sim\P} \big|g(\theta,\xi)\big|^2 &= \int_{\R^d\times\R^k}\mathrm{tr} \left\{(\nabla_\theta h)^T(\theta,x) \cdot \big(h(\theta,x)-y\big)^T \big(h(\theta,x) - y\big)\cdot \nabla_\theta h(\theta,x)\right\}\,\d\mu_{(x,y)}\\
	&= \int_{\R^d\times \R^k}\big|h(\theta,x)-y\big|^2 \,\mathrm{tr} \big\{(\nabla_\theta h)^T\nabla_\theta h\big\}(\theta,x)\,\d\mu_{(x,y)}\\
	&\leq \|\nabla_\theta h\|_{L^\infty}^2\int_{\R^d\times\R^k}\big|h(\theta,x)-y\big|^2\,\d\mu_{(x,y)}\\
	&= \|\nabla_\theta h \|_{L^\infty}^2\,L(\theta).
\end{align*}
As a Corollary, also 
\[
\E_{\xi\sim\mu} \big[|g(\theta,\xi) - \nabla L(\theta)|^2\big] \leq \E_{\xi\sim\mu} \big|g(\theta,\xi)\big|^2 \leq \|\nabla_\theta h \|_{L^\infty}^2\,L(\theta)
\]
since $\E_{\xi\sim\mu} g = \nabla L$.

{\bf Classification.} As in Example \ref{example cross-entropy}, we find that 
\[
\big|\nabla_1\ell\big(h(\theta,x), i\big)\big|^2 \leq 2\,\min\big\{1,\ell^2\big(h(\theta,x), i\big)\big\} .
\]
We set $\xi = (x,i)$ and estimate
\begin{align*}
g(\theta,\xi) &= (\nabla_1\ell)\big(h(\theta,x), y\big)\cdot \nabla_\theta h(\theta,x)
\end{align*}
by
\begin{align*}
\E_{\xi\sim\mu} \big[|g(\theta,x)|^2\big] &\leq \|\nabla_\theta h\|_{L^\infty}^2 \E_{\xi\sim\mu}\big[\big|\nabla_1 \ell\big(h(\theta,x), y\big)\big|^2\big]\\
	&\leq 2\, \|\nabla_\theta h\|_{L^\infty}^2 \,\E_{\xi\sim\mu}\big[\min\{1,\ell^2\big(h(\theta,x), y\big)\big\}\big]\\
	&\leq 2\, \|\nabla_\theta h\|_{L^\infty}^2 \,\E_{\xi\sim\mu}\big[\min\{1,\ell\big(h(\theta,x), y\big)\big\}\big]\\
	&\leq 2\, \|\nabla_\theta h\|_{L^\infty}^2 \,\min\{\E_{\xi\sim\mu}\big[\ell\big(h(\theta,x), y\big)\big], 1\big\}\\
	&= 2\, \|\nabla_\theta h\|_{L^\infty}^2 \,\min\{L(\theta), 1\}.
\end{align*}
\end{proof}

\begin{remark}\label{remark gradients linear}
Neural networks with $L$ layers can be decomposed as 
\[
h(\theta,x) = W\,\sigma\big(\hat h(\hat\theta,x)\big), \qquad \theta = (W, \hat \theta)
\]
where $\hat h(\hat \theta,\cdot) :\R^d\to \R^{M}$ is a neural network with $L-1$ layers and $\sigma$ is a nonlinear map between vector spaces. In particular
\[
\nabla_{\hat \theta} h(\theta,x) = W\,\sigma'\big(\hat h(\hat\theta,x)\big)\,\nabla_{\hat \theta}\hat h(\hat\theta,x),
\] 
i.e.\ also the gradient of $h$ with respect to deep layer weights is linear in the final linear weights. The derivative $\sigma'$ has to be interpreted as a diagonal matrix. By the chain rule, the linearity extends to the gradient of the loss function. The assumption that $|\nabla_\theta h| \leq C$ independently of $\theta$ and $x$ is therefore unrealistic in deep learning, but a weaker assumption like
\[
|\nabla_\theta h|(\theta,x) \leq C\sqrt{1+|\theta|^2}, \qquad \E\big[|g(\theta,\xi)|^2\big]\leq C\,f(\theta)\,(1+|\theta|^2)
\]
holds for two-layer neural networks with ReLU activation and deep networks (ResNets, multi-layer perceptra, ...) if the activation function $\sigma$ and its first derivative $\sigma'$ are bounded. 
\end{remark}

All statements above apply to both the overparametrized and underparametrized setting. A noise intensity which may scale with the objective function is a key feature of machine learning applications in general. The next result is specific to the overparametrized regime, where we show that the stochastic noise has low rank.

\begin{lemma}\label{lemma low rank}
Under the same conditions as Theorem \ref{theorem cooper}, we find that $\rk(\Sigma(\theta))\leq n$ for all $\theta\in\R^m$, independently of $k$.
\end{lemma}

\begin{proof}
If we set $L_i(\theta) = \big|h(\theta,x_i)-y_i\big|^2$, then 
\begin{align*}
\Sigma(\theta) &= \frac1n \sum_{i=1}^n \big(\nabla L_i - \nabla L\big) \otimes \big(\nabla L_i-\nabla L\big)(\theta)
\end{align*}
is a sum of $n$ rank $1$ matrices.
\end{proof}

\section{Stochastic gradient descent with ML noise in discrete time}\label{section discrete}

\subsection{Functions satisfying a \L{}ojasiewicz inequality}

It is easy to prove that if the noise in SGD is proportional to an objective function $f$ which satisfies a \L{}ojasiewicz condition, SGD reduces $f$ with small, but fixed positive step size $\eta>0$. The SGD scheme converges linearly, much like non-stochastic gradient descent. This is not a realistic model for machine learning, as \L{}ojasiewicz inequalities rule out the presence of all critical points which are not the global minimum, but the same tools can be used in more complex topics below.

\begin{theorem}\label{theorem linear convergence}\label{theorem lojasiewicz}
Assume the following.
\begin{enumerate}
\item $f:\R^m\to [0,\infty)$ satisfies the \L{}ojasiewicz inequality
\[
\Lambda\,f(\theta) \leq |\nabla f|^2(\theta),\qquad \Lambda>0.
\]
\item $f$ is $C^1$-smooth and $\nabla f$ satisfies the one-sided Lipschitz-condition
\[
\big(\nabla f(\theta) - \nabla f(\theta')\big) \cdot(\theta-\theta') \leq C_L\,|\theta-\theta'|^2 \qquad \forall\ \theta,\theta'\in \R^m.
\]
\item The family of gradient estimators $g:\R^m\times \Omega\to \R^m$ satisfies the following properties:
\begin{enumerate}
\item $\E_{\xi\sim\P} g(\theta,\xi) = \nabla f(\theta)$ for all $\theta\in \R^m$.
\item $\E_{\xi\sim\P}\big[|g(\theta,\xi) - \nabla f(\theta)|^2\big] \leq \sigma f(\theta)$.
\end{enumerate}
\item The initial condition $\theta_0$ is a random variable which satisfies the bound $\E \big[f(\theta_0)\big]<\infty$.
\item The random gradient selectors $\{\xi_t\}$ are iid with law $\P$ and independent of $\theta_0$.
\end{enumerate}
If $\theta_t$ is generated by the SGD iteration law $\theta_{t+1} = \theta_t - \eta\,g(\theta_t,\xi_t)$ and $0 < \eta < \frac{\Lambda}{\Lambda+\sigma}\,\frac2{C_L}$, then
\[
\E\big[f(\theta_t)\big] \leq \rho_\eta^t \,\E\big[f(\theta_0)\big], \qquad \text{where}\quad \rho_\eta = 1- \Lambda\,\eta + \eta^2 \,\frac{C_L(1+\sigma)}{2\Lambda} <1.
\]
Furthermore $\beta^tf(\theta_t)\to 0$ almost surely for every $\beta\in [1, \rho_\eta^{-1})$.
\end{theorem}

We note that similar results for SGD with non-standard noise bounds have been proved in the uniformly convex setting in \cite{stich2020error,stich2019unified}.

\begin{proof}
{\bf Expected objective value.}
Take $\theta_t, g_t$ to be fixed for now and $\theta_{t+1} = \theta_t - \eta \,g_t$. Then
\begin{align*}
f(\theta_{t+1}) - f(\theta_t) &= \int_0^1 \frac{d}{ds} f\big(\theta_t - \eta s\,g_t\big)\ds\\
	&= \int_0^1 \nabla f\big(\theta_t - s\,\eta g_t\big)\cdot \big(-\eta g_t\big) \ds\\
	&= \eta\int_0^1\frac{1}s \big[\nabla f\big(\theta_t - s\,\eta g_t\big) - \nabla f(\theta_t)\big] \cdot \big[\theta_t -s\eta\,g_t-\theta_t\big] - \nabla f(\theta_t) \cdot \eta\,g_t\,\ds\\
	&\leq \int_0^1 \frac1s\,C_L\big|\theta_t - s\eta\,g_t -\theta_t\big|^2 - \eta\, \nabla f(\theta_t)\cdot g_t \ds\\
	&= C_L\,\eta^2 |g_t|^2 \left( \int_0^1 s\ds\right) - \eta \,\nabla f(\theta_t)\cdot g_t.
\end{align*}
Now we consider $\theta_t, g_t= g(\theta_t,\xi_t)$ as the random variables they are in practice. We note that $\theta_{t}$ only depends on $\theta_0, \xi_0,\dots,\xi_{t-1}$ and is independent of $\xi_t$. Thus in expectation
\begin{align*}
\E \big[ \nabla f(\theta_t) \cdot g(\theta_t, \xi_t)\big] &= \E\big[ \E\big[\nabla f(\theta_t) \cdot g(\theta_t, \xi_t) \,\big|\,\theta_0,\xi_0,\dots\xi_{t-1}\big]\big]\\
	&= \E\big[ \nabla f(\theta_t)\cdot \E\big[g(\theta_t, \xi_t) \,\big|\,\theta_0,\xi_0,\dots\xi_{t-1}\big]\big]\\
	&= \E_{\theta_0,\xi_0,\dots,\xi_{t-1}} \big[|\nabla f|^2(\theta_t)\big]\\
\E\big[ |g_t|^2\big] &= \E\big[ |g_t - \nabla f(\theta_t)|^2 + 2 \,g_t\cdot \nabla f(\theta_t) - |\nabla f(\theta_t)|^2\big]\\
	&= \E\big[ |g_t - \nabla f(\theta_t)|^2\big] + \E\big[|\nabla f(\theta_t)|^2\big]
\end{align*}
by the tower identity for conditional expectations. Hence
\begin{align*}
\E\big[ f(\theta_{t+1}) - f(\theta_t)\big] &\leq \frac{C_L\,\eta^2}2\,\E\big[|g_t|^2\big] - \eta\,\E\big[ |\nabla f|^2(\theta_t)\big]\\
	&= \frac{C_L\,\eta^2}2 \E\big[|g_t-\nabla f(\theta_t)|^2\big] + \frac{C_L\,\eta^2}2 \E\big[ |\nabla f|^2(\theta_t)\big] -\eta\, \E\big[ |\nabla f|^2(\theta_t)\big]\\
	&\leq \frac{C_L\sigma\,\eta^2}2\,\E\big[f(\theta_t)\big] + \left( \frac{C_L\,\eta^2}2 - \eta\right)\,\E\big[ |\nabla f|^2(\theta_t)\big].
\end{align*}
If $\frac{C_L\,\eta^2}2 - \eta <0$, i.e.\ $0 < \eta < \frac{2}{C_L}$, we can estimate further that
\[
\E\big[ f(\theta_{t+1}) - f(\theta_t)\big] \leq \left(\frac{C_L\sigma\,\eta^2}2 + \Lambda\left(\frac{C_L\,\eta^2}2 - \eta\right)\right) \,\E\big[f(\theta_t)\big]
\]
or equivalently
\begin{align*}
\E\big[ f(\theta_{t+1})\big] &\leq \left(1 + C_L\frac{\sigma+\Lambda}2\,\eta^2 - \Lambda\eta\right) \E\big[f(\theta_t)\big].
\end{align*}
The prefactor $\rho_\eta:=1 + C_L\frac{\sigma+\Lambda}2\,\eta^2 - \Lambda\eta$ is smaller than $1$ if and only if 
\[
1 - \Lambda \eta + \eta^2 \,\frac{C_L(\Lambda+\sigma)}2 < 1 \qquad\LRa\qquad \eta \left(\frac{C_L(\Lambda+\sigma)}2\,\eta - \Lambda\right) <0 \qquad \LRa \qquad \eta \in \left(0, \frac{ 2\Lambda}{C_L\,(\Lambda+\sigma)}\right).
\]

{\bf Almost sure convergence.} The argument is standard for series with summable norms. Set 
\[
E_t(\eps) = \{\sup_{s\geq t}\beta^sf(\theta_s) \geq \eps\}.
\]
Then
\[
E_\infty(\eps) := \bigcap_{t\geq 0} E_t(\eps) = \big\{\limsup_{t\to\infty}\beta^tf(\theta_t) \geq \eps\big\}
\]
satisfies
\begin{align*}
\P\big(E_\infty(\eps)\big) &\leq \P\big( E_t(\eps)\big)\\
	&\leq \sum_{s=t}^\infty \P\big(\{\beta^sf(\theta_s)\geq \eps\}\big)\\
	&\leq \sum_{s=t}^\infty \frac{\E\big[\beta^s f(\theta_s)\big]}\eps\\
	&\leq \frac{1} \eps \sum_{s=t}^\infty \left(\beta\rho_\eta\right)^s\,\E\big[f(\theta_0)\big]\\
	&= \frac{(\beta\rho_\eta)^t}{1-\beta\rho_\eta}\,\frac{\E\big[f(\theta_0)\big]}\eps.
\end{align*}
By taking $t\to\infty$, we observe that $\P\big(E_\infty(\eps)\big) =0$, i.e.\ $\limsup_{t\to\infty}\beta^tf(\theta_t) <\eps$ almost surely. As this holds for all $\eps>0$, we find that $\beta^tf(\theta_t)\to 0$ almost surely.
\end{proof}

The optimal learning rate if the constants are known satisfies
\[
\frac{d}{d\eta} \left(1-\Lambda\eta + \frac{C_L(\Lambda+\sigma)}{2}\,\eta^2\right)=0 \qquad \Ra\qquad \eta = \frac{\Lambda}{C_L(\Lambda+\sigma)}, \quad \rho_\eta = 1-\frac{\Lambda^2}{2\,C_L(\Lambda+\sigma)}.
\]
In the noiseless case $\sigma =0$, this boils down to the well-known estimate for deterministic gradient descent.

\begin{corollary}\label{corollary convergence to minimum lojasiewicz}
Assume that $f, \eta$ and $g$ are like in Theorem \ref{theorem linear convergence}. Then $\theta_t$ converges to a random variable $\theta_\infty$ in $L^2$ and almost surely. The rate of convergence is
\[
\|\theta_t - \theta_\infty\|_{L^2(\P)} \leq\rho_\eta^{t/2} \,\frac{\eta\sqrt{2C_L+\sigma}\,\E\big[f(\theta_0)\big]}{1- \rho_\eta^{1/2}}
\] $f(\theta_\infty) \equiv 0$ almost surely.
\end{corollary}

\begin{proof}
Note that 
\[
\theta_t = \theta_0 + \sum_{i=1}^t\big(\theta_i - \theta_{i-1}\big) 
\]
and that 
\begin{align*}
\sum_{i=1}^\infty\big\|\theta_i - \theta_{i-1}\big\|_{L^2} = \eta\sum_{i=1}^\infty\big\|g(\theta_i,\xi_i)\|_{L^2} 
	\leq \eta\sum_{i=1}^t\sqrt{\E\big[|\nabla f(\theta_i)|^2\big] + \E\big[|g(\theta_i,\xi_i) - \nabla f(\theta_i)|^2\big]}.
\end{align*}
Recall that 
\[
|\nabla f(\theta)|^2 \leq 2C_L\, f(\theta)\qquad\forall\ \theta \in \R^m
\]
for non-negative functions with Lipschitz-continuous gradient (see e.g.\ Lemma \ref{lemma bounded gradient} in Appendix \ref{appendix local convergence}) and that
\[
\E\big[ |g(\theta_i, \xi_i) - \nabla f(\theta_i)|^2\big] \leq \sigma\,\E\big[f(\theta_i)\big].
\]
as above. Thus
\begin{align*}
\sum_{i=1}^\infty\big\|\theta_i - \theta_{i-1}\big\|_{L^2} &\leq \eta \sum_{i=1}^\infty \left[ \sqrt{\E\big[2C_L\,f(\theta_i)\big]+ \sigma\,\E\big[f(\theta_i)\big]}\right]\\
	&\leq \eta \sqrt{2C_L+\sigma}\,\sqrt{\E\big[f(\theta_0)\big]} \sum_{i=1}^\infty \rho_\eta^{1/2}\\
	&<\infty,
\end{align*}
i.e.\ the limit
\[
\theta_\infty := \lim_{t\to\infty} \theta_t = \theta_0 + \sum_{i=1}^\infty (\theta_i - \theta_{i-1})
\]
exists in the $L^2$-sense and
\[
\|\theta_\infty-\theta_t\|_{L^2} \leq \sum_{i=t}^\infty \|\theta_{i+1} - \theta_i\|_{L^2} \leq \eta \sqrt{2CL+\sigma}\sqrt{\E\big[f(\theta_0)\big]} \sum_{i=t}^\infty \rho_\eta^{1/2} = \eta \sqrt{2CL+\sigma}\sqrt{\E\big[f(\theta_0)\big]} \frac{\rho_\eta^{t/2}}{1-\rho_\eta^{1/2}}.
\]
 Since the increments are summable, the convergence is pointwise almost everywhere by the same argument as in Theorem \ref{theorem linear convergence}.
\end{proof}

\subsection{Convergence close to the minimum}

Any function which has a critical point that is not the global minimum does not satisfy a \L{}ojasiewicz inequality. However, if such an inequality holds close to the minimum and the initial condition is very close to the global minimum, then with high probability we converge to the minimum exponentially fast.

\begin{theorem}\label{theorem local convergence}
Assume that there exists $\eps>0$ such that the local \L{}ojasiewicz inequality
\[
\Lambda\,f(\theta)\leq |\nabla f(\theta)|^2
\]
holds on the set
\[
U(\eps) = \{\theta : f(\theta)<\eps\}
\] 
and that $\nabla f$ is $C_L$-Lipschitz continuous on $U(\eps)$. Then for every $\delta>0$ there exists $\eps'>0$ such that the following holds: If $\theta_0\in U(\eps')$ almost surely, then with probability at least $1-\delta$ we have $\theta_t\in U(\eps)$ for all $t\in \N$. 

Conditioned on the event that $\theta_t\in U(\eps)$ for all $t$, the estimate 
\[\showlabel
\lim_{t\to\infty} \beta^t\,f(\theta_t) =0
\]
holds almost surely for every
\[\showlabel
\beta\ \in [1,\rho_\eta^{-1}) \qquad\text{where }\rho_\eta = 1 - \Lambda\eta + \eta^2 \frac{C_L(\Lambda+\sigma)}{2\Lambda}
\]
if the learning rate satisfies
\[
0 < \eta < \frac{2\Lambda}{C_L(1+\sigma)}.
\]
\end{theorem}

The proof idea is as follows: In any iteration, we expect the objective value to decrease. While it may increase at times, it can only increase by small amounts as long as the objective value is small, and it is unlikely for errors to accumulate sufficiently for $f(\theta_t)$ to exceed $\eps$ if $f(\theta_0)\ll \eps$ is small enough. In the event that $f(\theta_t)<\eps$ for all $t\in\N$, we can follow along the lines of Theorem \ref{theorem linear convergence}. The rigorous proof is a variation of that of \cite[Theorem 4]{mertikopoulos2020almost} and given in Appendix \ref{appendix local convergence}. The main difference to the original is that the error is controlled by a decaying learning rate in \cite{mertikopoulos2020almost} and by the low noise intensity for low objective value in our context. If the set $U(\eps)$ has multiple connected components, of course the \L{}ojasiewicz inequality and Lipschitz-condition are only required to hold locally as the result is local in nature.

\begin{remark}
A version of Corollary \ref{corollary convergence to minimum lojasiewicz} holds also in this case with virtually the same proof, conditioned on the event that $\theta_t\in U(\eps)$ for all $t\in \N$.
\end{remark}

\subsection{On the global convergence of SGD with ML noise}

Whether SGD converges to a global minimum even with poor initialization if the target function does not satisfy a \L{}ojasiewicz inequality is quite delicate and requires strong assumptions. In general, this cannot be guaranteed.

\begin{example}
Consider the functions
\[
h_\alpha(x) = \frac{(x^2-1)^2}4 + \alpha x, \qquad f_\alpha(x) = h_\alpha(x) - \inf_{x'\in\R}h_\alpha(x').
\]
For $\alpha =0$, the function $f_\alpha$ has two minima of equal depth at $x=\pm 1$. If $\alpha$ is small but non-zero, the two local minima do not have equal depth. Assume that the gradient estimators are of the form
\[
g(x,\xi) = f_\alpha'(x) + \sqrt{\sigma \,f_\alpha(x)}\,\xi
\]
where $\xi$ is equal to $1$ or $-1$ with equal probability. If 
\[
\sqrt\sigma < \min_{1/4 < |x| < 1/2} \frac{|f_\alpha'(x)|}{\sqrt{f_\alpha(x)}}
\]
and $\eta$ is reasonably small, $x$ can never escape the potential well it started in.
\end{example}

The situation is different if the noise is unbounded and the objective function has particularly convenient properties not only at the set of global minimizers.

\begin{theorem}\label{theorem global convergence}
Let $f:\R^m\to [0,\infty)$ be a function such that
\begin{enumerate}
\item The set $N:= f^{-1}(0)$ is not empty.
\item $f$ is $C^1$-smooth and $\nabla f$ satisfies the one-sided Lipschitz condition
\[\showlabel\label{eq lipschitz global convergence}
\big(\nabla f(\theta) - \nabla f(\theta')\big) \cdot \big(\theta-\theta') \leq C_L |\theta-\theta'|^2.
\]
\item A \L{}ojasiewicz inequality holds on the set where $f$ is small, i.e.\ there exist $\eps, \lambda >0$ such that
\[\showlabel \label{eq lojasiewicz small}
f(\theta) < \eps \qquad\Ra\qquad \lambda\,f(\theta)\leq |\nabla f(\theta)|^2.
\]
\item A \L{}ojasiewicz inequality holds on the set where $f$ is large, i.e.\ there exist $S, \Lambda>0$ such that 
\[\showlabel \label{eq lojasiewicz large}
f(\theta) \geq S \qquad\Ra\qquad \Lambda\,f(\theta)\leq |\nabla f(\theta)|^2.
\]
\item $f$ grows uniformly away from its minimum, i.e.\ there exists $R>0$ such that 
\[\showlabel\label{eq low objective tight}
f(\theta)\leq  S \qquad\Ra\qquad \exists\ \theta' \text{ s.t. } f(\theta') = 0\text{ and } |\theta-\theta'|<R
\]
where $S$ is the same as in \eqref{eq lojasiewicz large}.
\end{enumerate}
Assume that the gradient noise satisfies the following:
\begin{enumerate}
\item The noise has ML type, i.e.\
\[
\E_\xi \big[ |g(\theta,\xi) - \nabla f(\theta)|^2\big] \leq \sigma\,f(\theta) \qquad\forall\ \theta.
\]
\item The noise is uniformly unbounded in the sense that
\[
g(\theta, \xi) = \nabla f(\theta) + \,\sqrt{\sigma\,f(\theta)}\,Y_{\theta,\xi}
\]
and there exists a continuous function $\psi:(0,\infty)^2\to(0,\infty)$ such that $Y$ satisfies
\begin{equation}\label{eq spread out noise}
\E Y_{\theta,\xi} =0, \qquad \P\big(Y_{\theta,\xi} \in B_r(\tilde\theta)\big) \geq \psi\big(|\tilde\theta |,r\big)>0 \qquad\forall\ \tilde\theta\in \R^m, r>0
\end{equation}
independently of $\theta$.
\end{enumerate}
Finally, assume that the learning rate satisfies 
\[
0<\eta< \min\left\{\frac{2\lambda}{C_L(\lambda +\sigma)}, \frac{2\Lambda}{C_L(\Lambda+\sigma)}\right\}.
\]
Assume that $\theta_0$ is an initial condition such that $\E\big[f(\theta_0)\big]<\infty$. Then the estimate
\[
\limsup_{t\to\infty} \frac{f(\theta_t)}{\rho^t} <\infty
\]
holds almost surely for every
\[\showlabel
\rho \in \left(\rho_\eta, 1\right), \qquad \rho_\eta:= 1 - \lambda\eta + \eta^2 \frac{C_L(\lambda+\sigma)}{2\lambda}.
\]
\end{theorem}

\begin{proof}[Proof sketch]
We use the \L{}ojasiewicz inequality on the set $\{f>S\}$ to show that SGD iterates visit the set $\{f\leq S\}$ infinitely often almost surely. For $\theta_t\in \{f\leq S\}$, there exists a uniformly positive probability that the $\theta_{t+1} \in \{f<\eps'\}$ due to the condition on the noise to be sufficiently `spread out' and on the set $\{f\leq S\}$ to be sufficiently close to $\{f=0\}$. Thus for infinite time, we visit the set $\{f<\eps'\}$ almost surely. 

Once we enter $\{f<\eps'\}$, we remain trapped in the set $\{f<\eps\}$ with high probability by Theorem \ref{theorem local convergence}, and in that case, we almost surely observe that $f\to 0$ at a linear rate. Even if we leave from $\{f<\eps\}$, we almost surely visit again. The probability to escape infinitely often vanishes, so we expect that $f(\theta_t)\leq \rho_\eta^{t-T}\,f(\theta_T)$ for all $t\geq T$ for almost every realization of SGD for some random time $T$.  
\end{proof}

The full proof is given in Appendix \ref{appendix global convergence}. With a variation of the proof of Corollary \ref{corollary convergence to minimum lojasiewicz}, the following can be obtained. Note that the rate of convergence $\rho_\eta$ depends only on the \L{}ojasiewicz constant on the set where $f$ is small. 

\begin{corollary}\label{corollary global convergence argument}
Assume that $f, \eta$ and $g$ are like in Theorem \ref{theorem global convergence}. Then $\theta_t$ converges to a minimizer of $f$ almost surely.
\end{corollary}

\begin{remark}
Our results suggest that, in a fairly general class of functions, SGD with ML noise converges to a global minimizer exponentially fast (or at least that the objective function decays exponentially along SGD iterates). However, it may take a very long time to reach the set in which exponential convergence is achieved, especially if the dimension $m$ of the parameter space and the co-dimension $m-n$ of the minimizer manifold $N$ are both high. 
\end{remark}

We give a few simple examples of situations in which Theorem \ref{theorem global convergence} can in fact be applied. Despite its weaknesses, we believe this mechanism to be a major driving factor behind the success of SGD in the machine learning of overparametrized neural networks.

\begin{remark}
If the function $f$ satisfies the conditions of Theorem \ref{theorem global convergence}, noise of the form
\[
g(\theta,\xi) = \nabla f(\theta) + \sqrt{\sigma\,f(\theta)}\,Y_\xi
\]
is admissible, where where $Y\sim \mathcal N(0,I)$ is standard Gaussian noise.
\end{remark}

\begin{example}
The following functions satisfy the conditions of Theorem \ref{theorem global convergence}, but do not satisfy a global \L{}ojasiewicz inequality:
\begin{enumerate}
\item $f(x,y) = \sin(x) + \sin(y)+2$. The set of minimizers is a lattice, $f$ is bounded, and every minimizer is non-degenerate.
\item $f_\eps(x) = x^2 + 2\eps\,\sin^2(x/\eps)$ for any fixed $\eps>0$. The only minimizer is at $x=0$. Far away from the origin, the oscillatory perturbation becomes negligible in both $f_\eps$ and $f_\eps'$.
\end{enumerate}
More generally, if $f$ is a perturbation of a quadratic form in the following sense, then Theorem \ref{theorem global convergence} applies:
\begin{itemize}
	\item $f\geq 0$.
	\item The set $\{f=0\}$ is a finite union of disjoint compact $C^2$-manifolds $N_1, \dots, N_k$ (potentially of different dimensions). If $\theta\in N_k$, then $D^2f(\theta)$ has rank $m-\dim(N_k)$ at $x$ and the smallest non-zero singular value of $D^2f$ is bounded away from zero.
	\item The perturbation becomes negligible at infinity, i.e. there exists a positive definite symmetric matrix $A$ such that
	\[
	\lim_{|\theta|\to\infty}\frac{f(\theta)}{\frac12\,\theta^TA\theta} = 1, \qquad \lim_{|\theta|\to\infty}\frac{|\nabla f(\theta) - A\theta|}{|A\theta|} = 0.
	\]
\end{itemize}
Also the first example could be generalized to small perturbations of periodic functions with non-degenerate minimizers.
\end{example}

\begin{example}
The following functions {\em do not} satisfy the conditions of Theorem \ref{theorem global convergence} since there are low energy points arbitrarily far away from the set of global minimizers:
\[
f_1, f_2:\R\to\R, \qquad f_1(x) = 1+ \sin^2(x) - \frac1{x^2+1}, \qquad f_2(x) = 1+ x^2 \sin^2(x) - \frac1{x^2+1}.
\]
In this one-dimensional setting, that problem could presumably be solved by appealing to the recurrence of random walks, but the situation is hopeless in analogous constructions in dimension three or higher.
\end{example}

Let us consider the conditions placed in Theorem \ref{theorem global convergence} in the context of deep learning.

\begin{remark}
\begin{itemize}
\item As the parameter gradient of a deep neural network with respect to deep layer coefficients is linear in the final layer coefficients (see Remark \ref{remark gradients linear}), the Lipschitz-condition \eqref{eq lojasiewicz small} cannot hold globally. If the activation function is smooth, it holds locally.
\item Generically, the Hessian of $L$ has maximal rank on $N$ due to Theorem \ref{theorem cooper}, so the local \L{}ojasiewicz inequality \eqref{eq lojasiewicz small} should hold at least locally on $\R^m$. 
\item Due to Lemma \ref{lemma no high energy critical points}, the objective function has no critical points where it is large. While \eqref{eq lojasiewicz large} may not be satisfied as such, a condition of this type does not violate the spirit of the minimization problem under consideration.
\item It is currently unknown to us whether the condition \eqref{eq low objective tight} that the set of low objective values is within a bounded tube around $N$ is generically satisfied or not. We offer the following rationale: Let $\theta$ such that $L(\theta)\leq S$. Choose $n' = \frac{m-nk}k$ additional points $x_{n+1}, \dots, x_{n+n'}$ (assuming that $n'$ is an integer) and set $\hat y_i = h(\theta,x_i)$ for $1\leq i \leq n+n'$. If the map
\[
\widehat \Phi:\R^m\to \R^{nk+n'k} = \R^m, \qquad \widehat\Phi(\theta) = \big(h(\theta, x_1), \dots, h(\theta, x_n)\big)
\]
has a Lipschitz-continuous inverse, then there exists $\theta'$ such that 
\begin{enumerate}
\item $\Phi(\theta', x_i) = y_i$ for $1\leq i\leq n$ and $\Phi(\theta', x_i) = \hat y_i$ for $n+1\leq i \leq n+n'$.
\item $\theta$ and $\theta'$ are somewhat close as \begin{align*}
|\theta-\theta'| &\leq C\,|y-y'| 
	= C\sqrt{\sum_{i=1}^n |y_i-\hat y_i|^2}
	=C\sqrt{\sum_{i=1}^n |y_i-h(\theta,x_i)|^2}\\
	& = C\sqrt{n\,L(\theta)} \leq C\sqrt{nS}.
\end{align*}
\end{enumerate}
Thus a condition of the type \eqref{eq low objective tight} may hold for nicely parametrized models, but the constant is expected to be somewhat large if the data set is large.

\item The type of noise specified in \eqref{eq spread out noise} is unrealistic in overparametrized learning models as it is omni-directional, whereas realistic noise is necessarily low rank due to Lemma \ref{lemma low rank}. The fact that the noise is `spread out' is needed to guarantee that we may randomly `jump' into the set $\{f<\eps\}$ from anywhere in the set $\{f\leq S\}$. Identifying more realistic geometric conditions with similar guarantees remains an open problem.
\end{itemize}
\end{remark}

\begin{remark}
Theorem \ref{theorem global convergence}, were it to apply, could be viewed as a negative result in the context of implicit regularization in machine learning. We can decompose the mean squared error population risk functional as
\begin{align*}
\Risk(h) &= \int_{\R^d} \big|h(x)-y\big|^2 \,\P(\d x\otimes \d y)\\
	&= \int_{\R^d} \big|h(x)-h^*(x)\big|^2 \,\P(\d x\otimes \d y) + \int_{\R^d} \big|h^*(x)-y\big|^2 \,\P(\d x\otimes \d y)\\
	&= \int_{\R^d} \big|h(x)-h^*(x)\big|^2 \,\P(\d x\otimes \d y) + C_0
\end{align*}
where $h^*(x) = \E[y|x]$ is the risk minimizer in the class of measurable functions and $C_0 = \Risk(h^*)$ is the minimum Bayes risk. If the distribution $\P$ admits any uncertainty in the output $y$ given observations $x$, then $C_0>0$. However, in overparametrized learning the empirical risk
\[
\widehat\Risk_n(h) = \frac1n\sum_{i=1}^n \big|h(x_i)-y_i\big|^2
\]
can be zero. Thus if $C_0$ is large and the parameters of a model are trained by SGD with small positive learning rate, then after a long time, we expect $\Risk$ and $\widehat\Risk_n$ to differ greatly at the parameters $\theta_t$. To avoid overfitting the training data, we therefore require an early stopping strategy. If the observations are virtually noiseless (which is the case for benchmark image classification problems), then this rationale may not apply, and SGD may perform well without early stopping.
\end{remark}

\section{A numerical illustration}\label{section numerical}

\begin{figure}
\includegraphics[width = 0.31\textwidth]{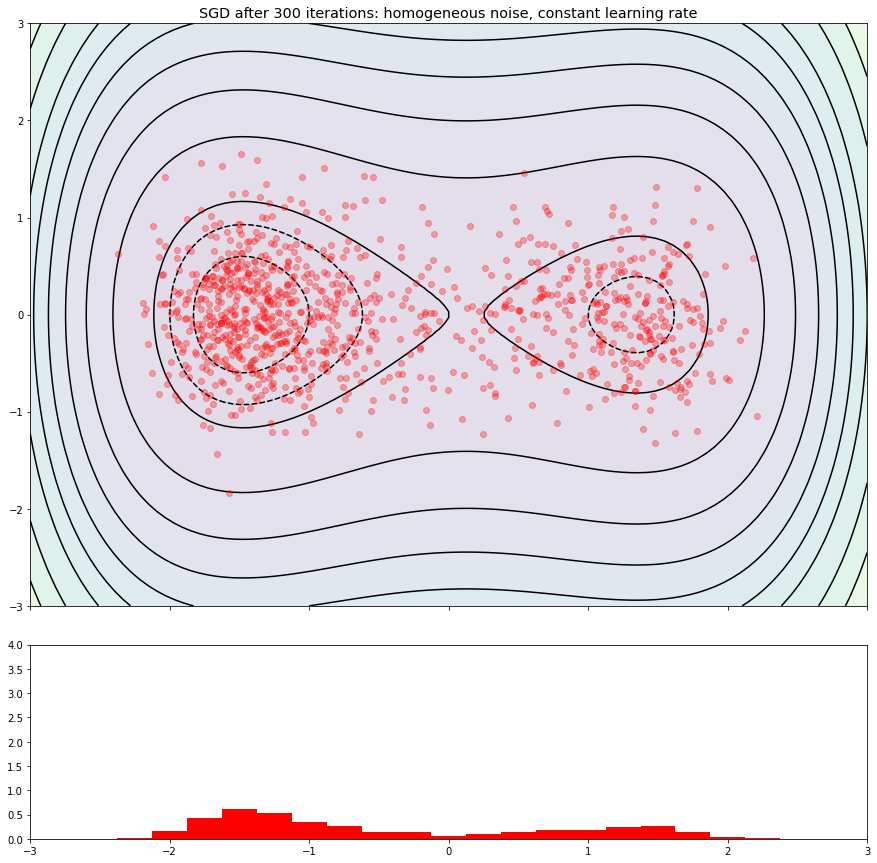}\hfill 
\includegraphics[width = 0.31\textwidth]{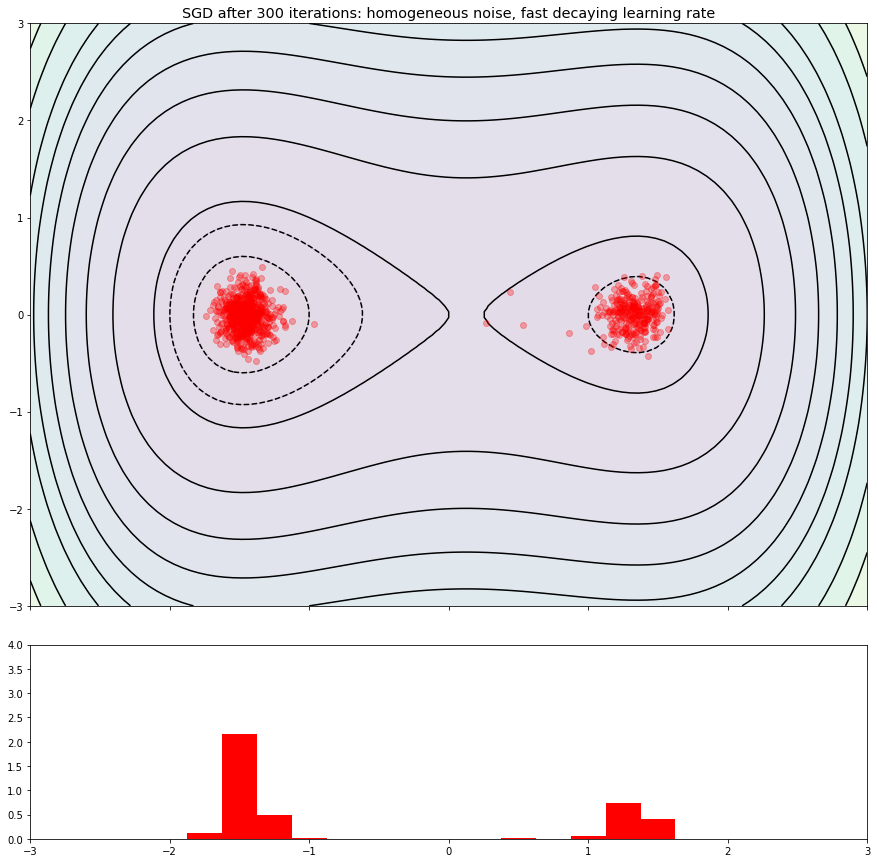}\hfill 
\includegraphics[width = 0.31\textwidth]{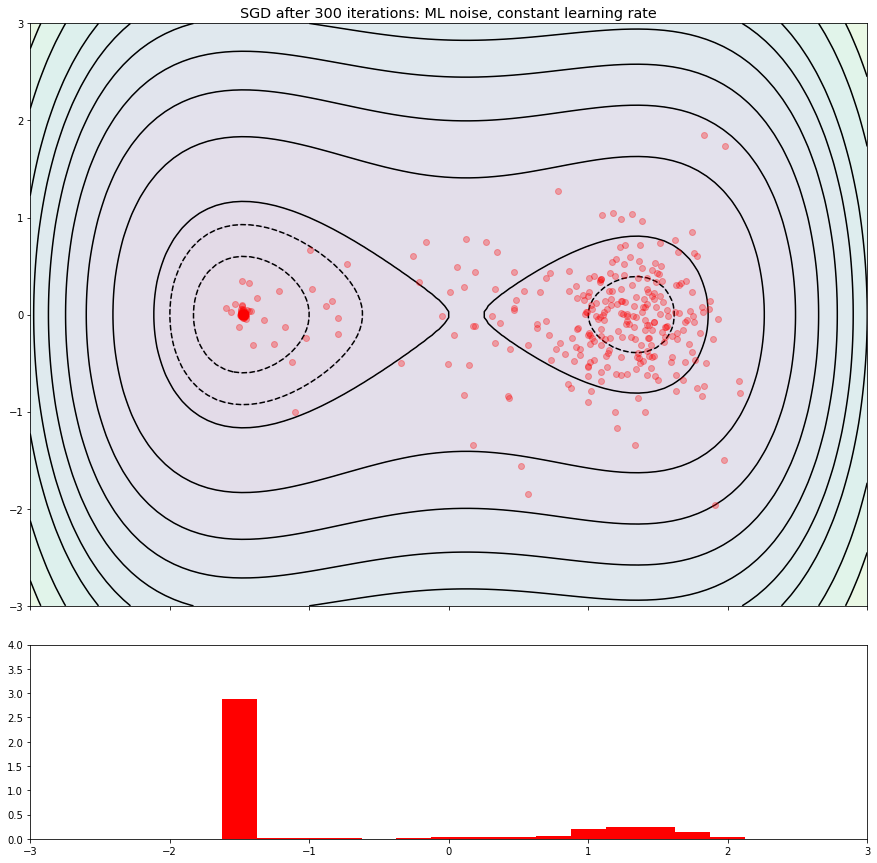}

\vspace{4mm}
\includegraphics[width = 0.31\textwidth]{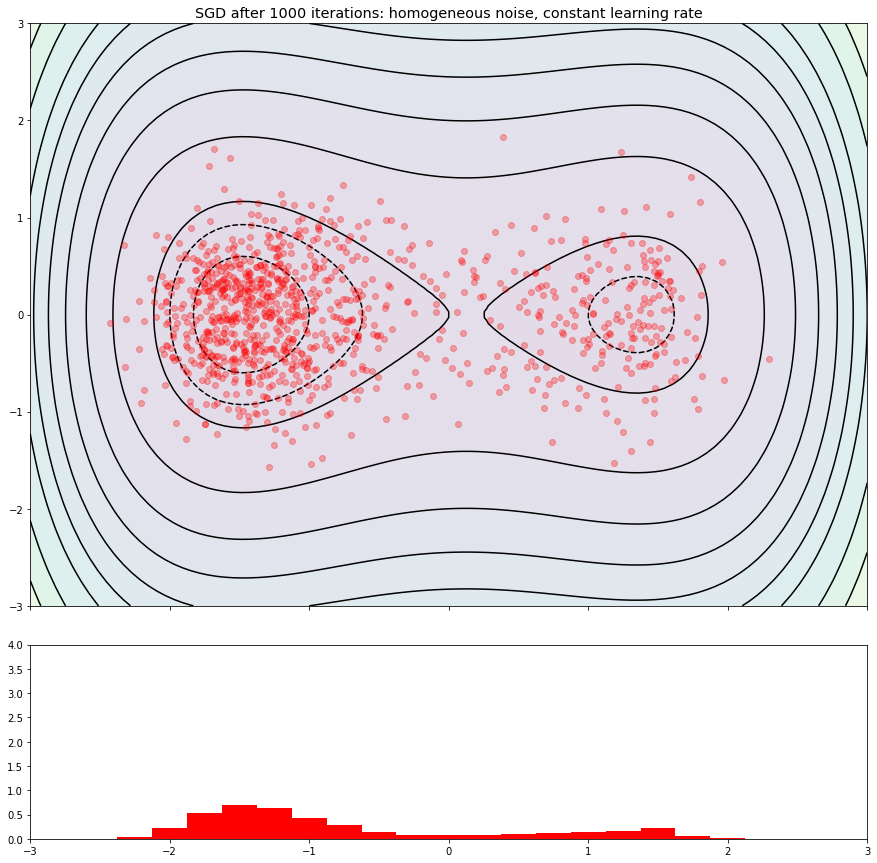}\hfill 
\includegraphics[width = 0.31\textwidth]{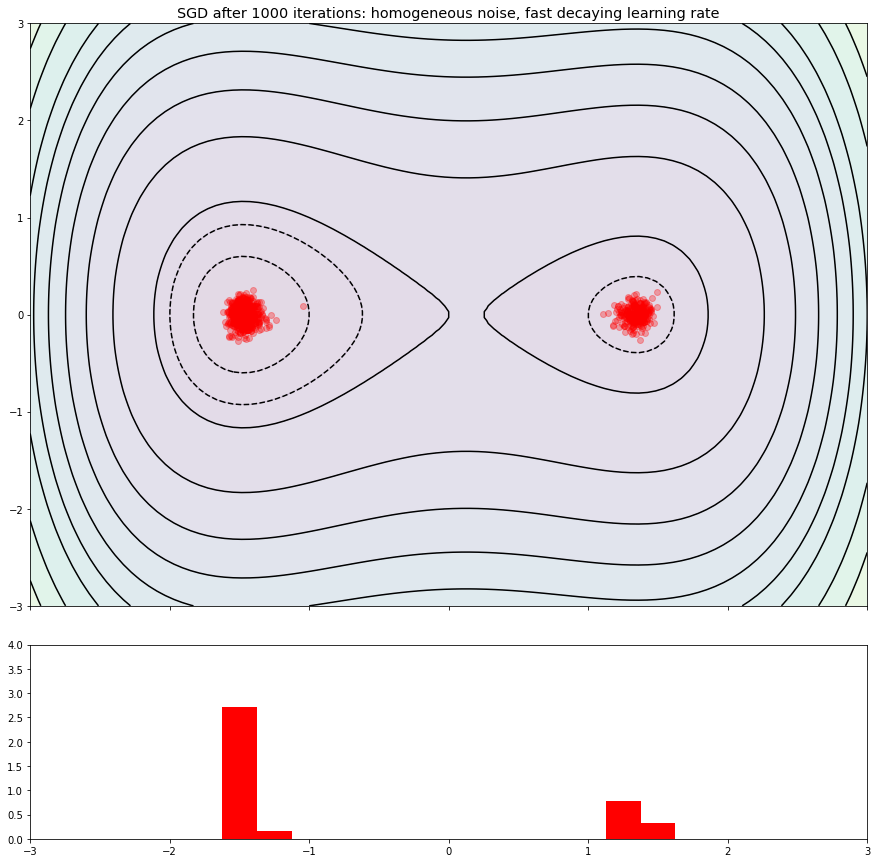}\hfill 
\includegraphics[width = 0.31\textwidth]{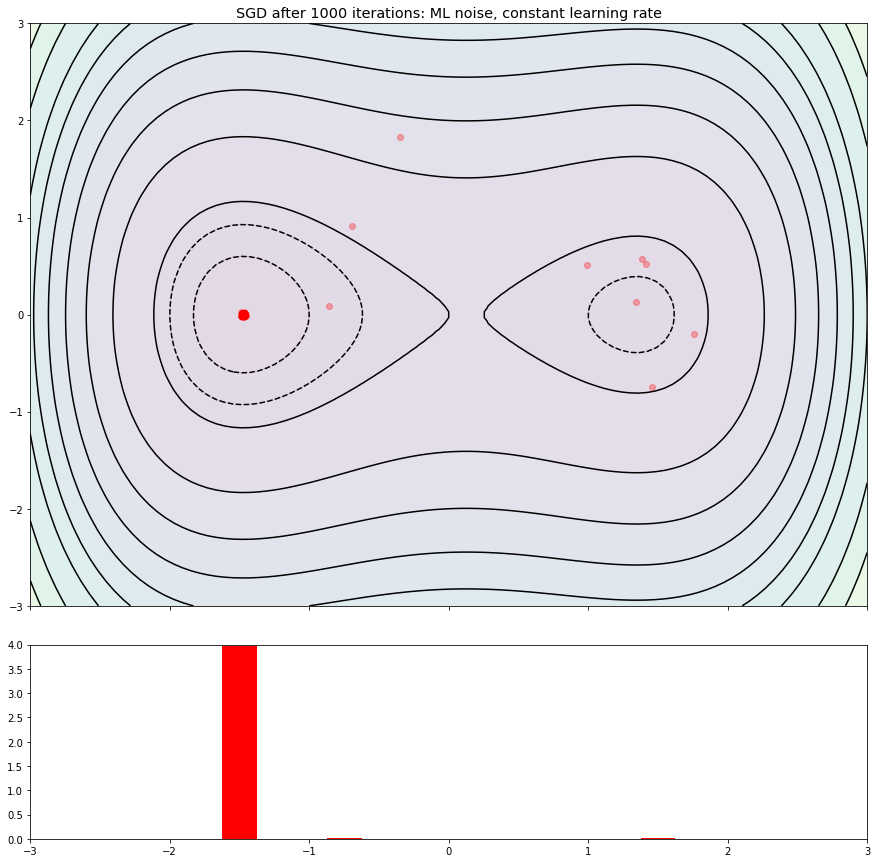}
\caption{SGD with homogeneous noise and constant learning rate (left) or decaying learning rate (middle) vs SGD with ML noise (right) after 300 iterations (top row) and after 1000 iterations (bottom row). For 1000 realizations of SGD, we mark the final position by a red dot in the energy landscape and plot a histogram density estimator of the $x$-coordinate below. 
\label{figure comparison}}
\end{figure}

We compare toy models for stochastic gradient descent
\begin{equation}\label{eq comparing sgds}
x_{n+1} = x_n - \eta \big(\nabla f(x_n) + \sigma Y_n\big), \qquad \tilde x_{n+1} = \tilde x_n - \eta \big(\nabla f(\tilde x_n) + \widetilde \sigma \sqrt{f(\tilde x_n)} \,Y_n\big)
\end{equation}
where $\eta$ is a learning rate, $\sigma, \widetilde \sigma>0$ control the noise level, and $Y_n$ is an is a standard Gaussian. For a fair comparison, $\sigma$ and $\widetilde\sigma$ have to be chosen such that the noise has the same magnitude at a certain point of interest.
Consider the function
\[
f:\R^2\to \R, \qquad f(x,y) = \frac{x^4}4 - x^2 +\alpha x + y^2 + c_\alpha
\]
where $\alpha \geq 0$ is a parameter and $c_\alpha$ is chosen such that $\inf_{(x,y)} f(x,y) = 0$. For any $\alpha\in(0,1]$, the function $f$ has two distinct local minima at $(x_-,0)$ and $(x_+,0)$ where $x_- < 0 < x_+$. Clearly, we have $f(x_-,0) < f(x_+,0)$. 

We considered 1000 realizations of SGD according to the schemes \eqref{eq comparing sgds} with $\alpha = 0.25$, learning rate $\eta = 0.05$ and noise $\widetilde\sigma = 4.0$. The parameter $\sigma$ was chosen such that the noise at the ridge between the minima of $f$ has the same intensity for both algorithms to give them equal opportunity to escape the local minimum. In all runs, the initial condition was $(x_0, y_0) = (3,0)$.

In Figure \ref{figure comparison}, we see that SGD with ML type noise and constant learning rate approaches the global minimum rapidly whereas SGD with homogeneous Gaussian noise and constant learning rate forms a cloud around the minima which has higher density at the global minimum. SGD with homogeneous noise and decaying learning rate forms more focussed clouds around the minima, but about 25\% of trajectories do not escape the local minimum. The decaying learning rate is chosen as 
\begin{equation}\label{eq learning rates}
\eta_t = t_0\,\frac{\eta}{t+t_1}, \qquad t_0 = 20, \quad t_1=5.
\end{equation}
We make similar observations if the learning rate decays of the same order but less rapidly with $t_0 =200$ and $t_1=50$, but with less focussed point clouds.

After 300 iterations, the point cloud of SGD with homogeneous noise and positive learning rate is comparable to that of SGD with ML noise at the local minimum. This seems to be due to the fact that both start closer to the local minimum and are equally likely to escape the local minimum due to our scaling of the noise.

Once a trajectory of SGD with ML noise enters the potential well of the global minimum, it is unlikely to escape, whereas SGD with homogeneous Gaussian noise exchanges particles back and forth between the two wells. Once within the potential well of the global minimum, a trajectory of SGD with ML noise converges to the global minimum rapidly. Therefore, after 1000 iterations, trajectories of SGD with ML noise are almost guaranteed to have found the global minimum, whereas trajectories of SGD with homogeneous noise are somewhat likely to be found in the potential well of the local minimum, independently of whether the learning rate is constant or decaying.

\section{Conclusion}

More realistic abstract models for the noise of stochastic gradient descent in machine learning may explain help explain some of the success that SGD has enjoyed in this specific non-convex optimization task. We believe our results are indicate why SGD often finds global minima/low loss local minima rather than positive loss local minima.

Our results give an indication that it may be admissible in these applications to leave the learning rate uniformly positive. This may be particularly relevant for online learning tasks, where new data is added to the model. Decaying learning rates significantly diminish the effect that new data can have in finite time.
The reason that the learning rate should be small in these optimization tasks is related both to the roughness of the loss landscape and the presence of stochastic noise. The estimates resemble non-stochastic gradient descent more closely than SGD with homogeneous noise estimates.

There are several pressing questions that were not explored in this work, both on the theoretical and practical side.

\begin{enumerate}
\item The biggest deficiency of our global convergence result is the reliance on omni-directional noise while realistic noise in machine learning SGD is low rank, at least for overparametrized models. Understanding the geometry of realistic noise and its impact on the convergence of SGD is an important open problem.

\item The objective functions we employ serve as toy models for the energy landscape of $L^2$-regression problems in deep learning. The energy landscape of classification problems and the associated gradient noise are quite different as noted in Example \ref{example loss functions} and Lemma \ref{lemma noise bounds}. Understanding the behavior of SGD in classification-like loss landscapes and noise models remains an open problem.

\item While our assumptions are fairly general in the context of optimization theory, they likely are too restrictive in the setting of deep learning. Understanding the geometry of deep learning problem, and whether they allow for similar results, is the next important step in this approach to the analysis of SGD. 

\item The Lipschitz-constant of the gradient of the objective function is not generally uniformly controlled over the parameter space. If a minimum is too steep, gradient descent may escape from it exponentially fast due to finite step size effects. This mechanism cannot be captured by continuum models and renders the positive step-size analysis invalid in the context of deep learning, unless we assume or enforce confinement to a bounded domain by other means.

\item In practice we do not use random selection SGD (choosing a batch of data samples randomly from the training set), but random pass SGD (passing through the entire training set batch by batch before repeating the same data point). The random directions in consecutive iterations are therefore not truly iid, and the noise may be less oscillatory in practice than random selection estimation would suggest. We believe the impact of this difference to be negligible for large data sets, but a rigorous connection has not been established to the best of our knowledge.

\item Typically, advanced optimizers like gradient descent with momentum, Nesterov's accelerated gradient descent or ADAM are used  in deep learning with stochastically estimated gradients. It remains to be seen whether noise of ML type allows for stronger estimates and global convergence guarantees also in that setting.

\item The main goal of this article was to understand SGD in toy models inspired by problems of deep learning. Another interesting direction is whether artificially perturbing (exact or estimated) gradients by omni-directional noise of ML type can improve the convergence of a gradient descent type optimization algorithm.
\end{enumerate}

\bibliographystyle{../../alphaabbr}
\bibliography{../../NN_bibliography}

\appendix

\section{On the non-convexity of objective functions in deep learning}\label{appendix non-convexity}

Under general conditions, energy landscapes in machine learning regression problems {\em have to} be non-convex. The following result follows from Theorem \ref{theorem cooper}.

\begin{corollary}
Assume that $h:\R^m\times\R^d\to\R$ is a parameterized function model, which is at least $C^2$-smooth in $\theta$ for fixed $x$. Let $L_y(\theta) = \frac1n\sum_{i=1}^n \big(h(\theta,x_i) - y_i\big)^2$ where $y=(y_1,\dots,y_n)\in\R^n$. 

\begin{enumerate}
\item If $L_y$ is convex for every $y\in\R^n$, the map $\theta\mapsto h(\theta,x_i)$ is linear for all $x_i$.

\item Assume that $L_y(\theta) =0$ and that there exists $1\leq j \leq n$ such that $D^2h(\theta,x_j)$ has rank $k>n$. Then for every $\eps>0$, there exists $y'\in \R^m$ such that $|y-y'|<\eps$ and $D^2L_{y'}(\theta)$ has a negative eigenvalue.

\item Assume that $L_y(\theta) =0$ and that there exists $1\leq j \leq n$ such that $D^2h(\theta,x_j)$ has rank $k>n$. Assume furthermore that the gradients $\nabla_\theta h(\theta,x_1), \dots, \nabla_\theta h(\theta, x_n)$ are linearly independent. Then for every $\eps>0$, there exists $\theta'\in \R^m$ such that $|\theta-\theta'|<\eps$ and $D^2L_{y}(\theta')$ has a negative eigenvalue.
\end{enumerate}
\end{corollary}

The first statement is fairly weak, as the proof requires us to consider the convexity of $L_y$ far away from the minimum. The second statement shows that even close to the minimum, $L_y$ can be non-convex if the function model is sufficiently far from being `low dimensional' -- this statement concerns perturbations in $y$. The third claim is analogous, but concerns perturbations in $\theta$. It is therefore stronger, as it shows that $L_y$ is not convex in any neighborhood of a given point in the set of minimizers. 

\begin{proof}
{\bf First claim.}
Compute 
\begin{align*}
\nabla L_y(\theta) &= \frac2n \sum_{i=1}^n \big(h(\theta,x_i) - y_i\big)\,\nabla_\theta h(\theta,x_i)\\
D^2 L_y(\theta) &= \frac2n \sum_{i=1}^n \left[\nabla h(\theta, x_i) \otimes \nabla h(\theta, x_i) + \big(h(\theta,x_i) - y_i\big)\,D^2h(\theta, x_i)\right].
\end{align*}
If $D^2h(\theta, x_i)\neq 0$, there exists $v\in \R^d$ such that $v^T D^2h(\theta,x_i)v \neq 0$ since the Hessian matrix is symmetric. Thus
\[
\liminf_{|y_i|\to\infty} v^T \,D^2L_y(\theta)\,v = - \lim_{|y_i|\to \infty} y_i\,v^T D^2h(\theta,x_i)\,v = \pm\infty,
\]
meaning that $L_y$ cannot be convex. Thus if $L_y$ is convex for all $y$,  by necessity $D^2h(\theta, x_i)\equiv 0$ for all $i$, i.e.\ $\theta\mapsto h(\theta,x_i)$ is linear.

{\bf Second claim.} Set $y_j = h(\theta, x_j)$ for $j\neq i$ and $y_i = h(\theta,x_i) \pm \eps/2$. Then by a simple result in linear algebra (see Lemma \ref{lemma linear algebra} below), the matrix
\[
D^2L_y(\theta) = \frac2n \sum_{i=1}^n \left[\nabla h(\theta, x_i) \otimes \nabla h(\theta, x_i) \right] + \frac\eps2\,D^2h(\theta, x_i)
\]
has a negative eigenvalue.

{\bf Third claim.} Let $V= \mathrm{span}\{\nabla f(\theta,x_j) : j\neq i\}$. Since the set of gradients is linearly independent, there exists $v$ such that $v\bot V$ but $v^T \nabla f(\theta, x_i) \neq 0$. Thus to leading order
\[
D^2L_y(\theta + tv) = D^2L_y(\theta) = \frac2n \sum_{i=1}^n \left[\nabla h(\theta, x_i) \otimes \nabla h(\theta, x_i) \right] + t \big(\nabla f(\theta,x_i)\cdot v\big) D^2h(\theta, x_i) + o(t).
\]
As before, we conclude that $D^2L_y(\theta+tv)$ has a negative eigenvalue if $t>0$ is small enough.
\end{proof}

\begin{remark}
If $\theta\mapsto h(\theta,x_i)$ is linear, then $D^2h(\theta,x_i)$ vanishes, i.e.\ $\mathrm{rk}(D^2h) =0$. The large discrepancy between the conditions ensuring local convexity and global convexity is in fact necessary. Consider $h(\theta, x_i) = \theta_i$ for all $i\geq 1$ and
\[
h(\theta, x_1) = \theta_1 + \eps \,\psi(\theta_1,\dots,\theta_n)\qquad \Ra\quad D^2L_y(\theta) = \begin{pmatrix} I_{n\times n} + \eps\,\big(h(\theta,x_1)-y_1\big)\,D^2\psi &0_{n\times (m-n)}\\ 0_{(m-n)\times n} & 0_{(m-n)\times (m-n)}\end{pmatrix}.
\]
If $D^2\psi$ is bounded and $y$ is fixed, then for every $R>0$ we can choose $\eps$ so small that $L_y$ is convex on $B_R(0)$. Note that the rank of $D^2\psi$ is at most $n$ in this example.
\end{remark}

We prove a result which we believe to be standard in linear algebra, but have been unable to find a reference for.

\begin{lemma}\label{lemma linear algebra}
Let $A\in\R^{m\times m}$ be a symmetric positive definite matrix of rank at most $n<m$. Let $B$ be a symmetric matrix of at least $n+1$. Then, for every $s>0$ at least one of the matrices $A+sB$ or $A-sB$ has a negative eigenvalue.
\end{lemma}

\begin{proof}
Without loss of generality, we assume that $A = \diag(\lambda_1,\dots, \lambda_n, 0, \dots, 0)$. 

{\bf First case.} The Lemma is trivial if there exists $v\in \mathrm{span}\{e_{n+1},\dots, e_m\}$ such that $v^TBv\neq 0$ since then
\[
v^T\big(A+sB\big) v = s\,v^TBv
\]
is a linear function.

{\bf Second case.} Assume that $v^TBv =0$ for all $v\in \mathrm{span}\{e_{n+1}, \dots, e_m\}$. Since $B$ has rank $n+1$, there exists an eigenvector $w$ for a non-zero eigenvalue $\mu$ of $B$ such that $w\notin \mathrm{span}\{e_1,\dots, e_n\}$. Without loss of generality, we may assume that $e_{n+1} \cdot w >0$. Consider
\begin{align*}
g_s(t)&:= \big(e_{n+1} + tv\big)^T \big(A+sB\big) \big(e_{n+1}+tv\big)\\
g_s(0) &= e_{n+1}^T (A+sB)e_{n+1}\\
	&=0\\
g'_s(0) &= 2\,v^T\big(A+sB\big)e_{n+1}\\
	&= 2 v^T(A e_{n+1}) + 2s\,e_{n+1}^T Bv\\
	&= 0 + 2\mu s\,e_{n+1}^Tv\\
	&\neq 0.
\end{align*}
Thus, we choose the correct sign for $t$ depending on $\mu$ and $s$, then
\[
\big(e_{n+1} + tv\big)^T \big(A+sB\big) \big(e_{n+1}+tv\big)<0.
\] 
In this situation, $B$ is indefinite and the sign of $s$ does not matter.
\end{proof}

\section{Auxiliary observations on objective functions and \L{}ojasiewicz geometry}

\begin{lemma}\label{lemma bounded gradient}
Assume that $f:\R^m\to\R$ is a non-negative function and $\nabla f$ is Lipschitz continuous with constant $C_L$. Then 
\[
|\nabla f(\theta)|^2 \leq 2C_L\, f(\theta)\qquad\forall\ \theta \in \R^m.
\]
\end{lemma}

\begin{proof}
Take $\theta\in \R^m$. The statement is trivially true at $\theta$ if $\nabla f(\theta) =0$, so assume that $\nabla f(\theta)\neq 0$. Consider the auxiliary function
\[
g(t) = f\big(\theta - t\,\nu\big)\qquad\text{where }\nu = \frac{\nabla f(\theta)}{|\nabla f(\theta)|}.
\]
Then $g'(0) = - \nu \cdot \nabla f(\theta) = |\nabla f(\theta)|$ and
\[
\big|g'(t) - g'(0)\big| = \big|\big(\nabla f(\theta - t\nu)- \nabla f(\theta)\big)\cdot \nu\big| \leq c_L\big|\theta - t\nu-\theta\big| = c_Lt.
\]
Thus 
\begin{align*}
g(t) &= g(0) + \int_0^t g'(s)\ds
	 \leq f(\theta) + \int_0^t - |\nabla f(\theta)| + C_Ls\ds
	=f(\theta) - |\nabla f(\theta)|t + \frac{C_L}2 t^2.
\end{align*}
The bound on the right is minimal for $t = -\frac{|\nabla f(\theta)|}{C_L}$ when
\[
f(\theta) - |\nabla f(\theta)|t + \frac{C_L}2 t^2 = f(\theta) - \frac{|\nabla f(\theta)|^2}{2\,C_L}.
\]
Since $f\geq 0$ also $g\geq 0$, so $ f(\theta) - \frac{|\nabla f(\theta)|^2}{2C_L} \geq0$.
\end{proof}

\begin{remark}
In particular, If $f$ satisfies a \L{}ojasiewicz inequality and has a Lipschitz-continuous gradient, then
\begin{equation}\label{eq grad squared and objective}
\Lambda\,f \leq |\nabla f|^2 \leq 2C_L \,f.
\end{equation}
\end{remark}

We show that the class of objective functions which can be analyzed by our methods does not include loss functions of cross-entropy type under general conditions.

\begin{corollary}
Assume that $f:\R^m\to\R$ is a $C^1$-function such that
\begin{itemize}
\item $f\geq 0$ and
\item \eqref{eq grad squared and objective} holds.
\end{itemize}
Then there exists $\bar\theta\in\R^m$ such that $f(\bar\theta)=0$.
\end{corollary}

\begin{proof}
Choose $\theta_0 \in \R^m$ and consider the solution of the gradient flow equation
\[
\begin{pde}
\dot\theta &= - \nabla f(\theta) &t>0\\
\theta &=\theta_0 &t=0.
\end{pde}
\]
Then
\[
\frac{d}{dt} f(\theta(t)) = -|\nabla f|^2(\theta(t)) \leq - \Lambda\,f(\theta(t))\qquad \Ra\qquad \frac{d}{dt} \log(f(\theta(t)))= \frac{\frac{d}{d t} f(\theta(t))}{f(\theta(t))} \leq - \Lambda
\]
and thus $f(\theta(t)) \leq f(\theta(0))\,e^{-\Lambda t}$. Furthermore
\begin{align*}
\big|\theta(t_2)-\theta(t_1)\big| &\leq \int_{t_1}^{t_2}|\nabla f|(\theta(s))\,\ds\\
	&\leq \int_{t_1}^{t_2}\sqrt{2C_L\,f(\theta(s))}\,\ds\\
	&\leq \sqrt{2C_L\,L(\theta(0))}\int_{t_1}^{t_2} e^{-\Lambda s/2}\ds\\
	&= \frac{2\sqrt{2C_L\,L(\theta(0))}}{\Lambda}\big[e^{-\Lambda t_1/2} - e^{-\Lambda t_2/2}\big]
\end{align*}
whence we find that $\theta(t)$ converges to a limiting point $\theta_\infty$ as $t\to\infty$. By the continuity of $f$, we find that
\[
f(\theta_\infty) = \lim_{t\to\infty}f(\theta(t)) = 0.
\]
\end{proof}

\section{Proof of Theorem \ref{theorem local convergence}: Local Convergence}\label{appendix local convergence}

We split the proof up over several lemmas. Our strategy follows along the lines of \cite[Appendix D]{mertikopoulos2020almost}, which in turn uses methods developed in \cite{hsieh2019convergence,hsieh2020explore}. We make suitable modifications to account for the fact that the smallness of noise comes from the fact that the values of the objective function are low, not that the learning rate decreases. Furthermore, we have slightly weaker control since we do not impose quadratic behavior with a strictly positive Hessian at the minimum, but only a \L{}ojasiewicz inequality and Lipschitz continuity of the gradients. 

While weaker conditions may hold for the individual steps of the analysis, we always assume that the conditions of Theorem \ref{theorem local convergence} are met for the remainder section. We decompose the gradient estimators as 
\[
g(\theta, \xi) = \nabla f(\theta) + \,\sqrt{\sigma\,f(\theta)}\,Y_{\theta,\xi}, \qquad \E_\xi\big[Y_{\theta,\xi}\big] =0, \qquad \E_\xi\big[|Y_{\theta,\xi}|^2\big] \leq 1
\]
and interpolate $\theta_{t+s} = \theta_t - s\eta\,g(\theta_t,\xi_t)$ for $s\in [0,1]$. With these notations, we can estimate the change of the objective in a single time-step as
\begin{align*}
f(\theta_{t+1})-f(\theta_t) &= \int_0^1 \frac{d}{ds} f\big(\theta_t - s\eta\,g(\theta_t,\xi_t)\big)\ds\\
	&= -\int_0^1 \nabla f(\theta_{t+s}) \cdot \eta\,g(\theta_t,\xi_t)\ds\\
	&= -\eta\int_0^1 \nabla f(\theta_t) \cdot \big[\nabla f(\theta_t) + \sqrt{\sigma f(\theta_t)}\,Y_{\theta_t,\xi_t}\big] + \big[\nabla f(\theta_{t+s}) - \nabla f(\theta_t) \big]\cdot g(\theta_t,\xi_t)\ds\\
	&\leq - \eta\,|\nabla f(\theta_t)|^2 + \eta\,\sqrt{\sigma\,f(\theta_t)} \nabla f(\theta_t)\cdot Y_{\theta_t,\xi_t} + C_L \eta\,\int_0^1|\theta_{t+s}-\theta_t|\, |g(\theta_t,\xi_t)|\ds\\
	&\leq - \eta\,|\nabla f(\theta_t)|^2 + \eta\,\sqrt{\sigma\,f(\theta_t)} \nabla f(\theta_t)\cdot Y_{\theta_t,\xi_t} + C_L \eta \int_0^1s\,|g(\theta_t,\xi_t)|^2\ds\\
	&= - \eta\,|\nabla f(\theta_t)|^2 + \eta\,\sqrt{\sigma\,f(\theta_t)} \nabla f(\theta_t)\cdot Y_{\theta_t,\xi_t}  + \frac{C_L\eta^2}2 |g(\theta_t,\xi_t)|^2\\
	&= - \eta\,|\nabla f(\theta_t)|^2 + \eta\,\sqrt{\sigma\,f(\theta_t)} \nabla f(\theta_t)\cdot Y_{\theta_t,\xi_t}  + \frac{C_L\eta^2}2 \big|\nabla f(\theta_t)+ \sqrt{\sigma \,f(\theta_t)}\,Y_{\theta_t,\xi_t}\big|^2\\
	&= \left(\frac{C_L\eta}2-1\right)\eta\,|\nabla f(\theta_t)|^2 + \left(1+ C_L\eta\right)\sqrt{\sigma}\,\eta\, \sqrt{f(\theta_t)}\,\nabla f(\theta_t)\cdot Y_{\theta_t,\xi_t} + \frac{C_L\sigma\,\eta^2}2\,\big|\sqrt{f(\theta_t)}\,Y_{\theta_t,\xi_t}\big|^2\\
	&= -\widetilde\eta \,|\nabla f(\theta_t)|^2 + \hat \eta\,\nabla f(\theta_t) \cdot \sqrt{f(\theta_t)}\,Y_{\theta_t,\xi_t} + \bar\eta^2 \big|\sqrt{f(\theta_t)}\,Y_{\theta_t,\xi_t}\big|^2
\end{align*}
where 
\[
\widetilde\eta = \left(\frac{C_L\eta}2-1\right)\eta, \qquad \hat \eta = \sqrt{\sigma}\,\left(1+ C_L\eta\right)\,\eta, \qquad \bar\eta = \sqrt{\frac{C_L\sigma}2} \,\eta.
\]
All three variables scale like $\eta$ and the difference between them can be ignored for the essence of the arguments. 
As usual, we denote by $\F_t$ the filtration generated by $\theta_0, \xi_0, \dots, \xi_{t-1}$, with respect to which $\theta_t$ is measurable.
We note that $\nabla f(\theta_t) \cdot \sqrt{f(\theta_t)}\,Y_{\theta_t,\xi_t}$ is a martingale difference sequence with respect to $\F_t$ since
\begin{align*}
\E\big[\nabla f(\theta_t) \cdot \sqrt{f(\theta_t)}\,Y_{\theta_t,\xi_t} | \F_t\big] &= \nabla f(\theta_t) \cdot \sqrt{f(\theta_t)}\,\E\big[Y_{\theta_t,\xi_t} |\F_t\big]=0.
\end{align*}
To analyze the $\theta_t$ over several time steps, we define the cumulative error terms
\begin{align*}
M_t &= \hat \eta\sum_{i=0}^{t} \nabla f(\theta_t) \cdot \sqrt{f(\theta_t)}\,Y_{\theta_t,\xi_t}\\
S_t &= \bar \eta^2 \sum_{i=0}^{t} \big|\sqrt{f(\theta_t)}\,Y_{\theta_t,\xi_t}\big|^2\\
R_t &= M_t^2 +S_t.
\end{align*}
We furthermore define the events
\begin{align*}
\Omega_t(\eps') &= \{ f(\theta_i)<\eps'\text{ for all }0\leq i\leq t\} &&= \{\text{``objective remains small''}\}\\
E_t(r) &= \{R_t < r \text{ for all } 0\leq i \leq t\} &&= \{\text{``noise remains small until time $t$''}\}\\
\widetilde E_t(r) &= E_{t-1}(r)\setminus E_t(r) &&= \{\text{``noise exceeds threshold in $t$-th step''}\}.
\end{align*}
for $\eps'>0$ and $0<r<1$. The sets $\widetilde E_t$ are useful as they allow us to estimate the measure of $E_t^c = \bigcup_{i=0}^t \widetilde E_i$, where all sets in the union are disjoint.

\begin{lemma}
The following are true.
\begin{enumerate}
\item $\Omega_{t+1}\subseteq \Omega_t$ and $E_{t+1}\subseteq E_t$.
\item If $\theta_0\in \Omega_0(\eps')$ for $\eps'<\eps$, then $E_{t-1}(r) \subseteq \Omega_t(\eps)$ if $\eps' +r + \sqrt{r} < \eps$.
\item Under the same conditions, the estimate
\begin{equation}\label{eq weird estimate}
\E\big[ R_t\,1_{E_{t-1}}\big] \leq \E\big[R_{t-1}\,1_{E_{t-2}}\big] + C\sigma \big(2c_L\eps +1\big)\eta^2\, \E\big[ 1_{E_{t-1}} f(\theta_t)\big] - r \,\P\big(\widetilde E_t\big)
\end{equation}
holds, where the constant $C>0$ incorporates the factors between $\hat \eta$, $\tilde \eta$ and $\bar \eta$.
\end{enumerate}
\end{lemma}

\begin{proof}
The first claim is trivial. 

{\bf Second claim.} Recall that $\theta_0$ is initialized in $\Omega(\eps')\subseteq \Omega(\eps)$. In particular $\Omega_0 = E_{-1} = \Omega$ is the entire probability space since the sum condition for $E_{-1}$ is empty. We proceed by induction. 

Assume that $\omega\in E_{t}$. Then in particular $\omega\in E_{t-1}$, so $\omega \in \Omega_{t}$ by the induction hypothesis. Thus it suffices to show that $f(\theta_t) < \eps$, i.e.\ to focus on the last time step. A direct calculation yields
\begin{align*}
f(\theta_t)&= f(\theta_0) + \sum_{i=1}^t\big[ f(\theta_i) - f(\theta_{i-1})\big]\\
	&\leq f(\theta_0) + M_t + S_t\\
	&\leq \eps' + \sqrt{R_t} + R_t\\
	&\leq \eps' + r + \sqrt{r}\\
	&< \eps.
\end{align*}

{\bf Third claim.} A simple algebraic manipulation shows that 
\begin{align*}
\E\big[R_t1_{E_{t-1}}\big] &= \E\big[R_{t-1}1_{E_{t-1}}\big] + \E\big[(R_t-R_{t-1})\,1_{E_{t-1}}\big]\\
	&=  \E\big[R_{t-1}1_{E_{t-2}}\big] - \E\big[R_{t-1}1_{\widetilde E_{t-2}}\big] + \E\big[(R_t-R_{t-1})\,1_{E_{t-1}}\big]\\
	&\leq \E\big[R_{t-1}1_{E_{t-2}}\big] + \E\big[(R_t-R_{t-1})\,1_{E_{t-1}}\big] - r\,\P(\widetilde E_{t-2})
\end{align*}
since $R_{t-1}\geq r$ on $\widetilde E_{t-2}$. We recall that
\begin{align*}
R_t &= M_t^2 + S_t\\
	&= R_{t-1} + \underbrace{\hat \eta \,M_{t-1} \cdot \nabla f(\theta_t) \cdot \sqrt{f(\theta_t)}\,Y_{\theta_t,\xi_t}}_{=: (I)} + \hat\eta^2 \underbrace{\big|\nabla f(\theta_t) \cdot \sqrt{f(\theta_t)}\,Y_{\theta_t,\xi_t}\big|^2}_{=: (II)} + \bar\eta^2\underbrace{\big|\sqrt{f(\theta_t)}\,Y_{\theta_t,\xi_t}\big|^2}_{(III)}
\end{align*}
and that
\begin{align*}
 \E\big[ (I)\cdot 1_{E_{t-1}} \big] &= \E\big[ 1_{E_{t-1}}\,M_{t-1} \cdot \nabla f(\theta_t) \cdot \sqrt{f(\theta_t)}\,\E\big[Y_{\theta_t,\xi_t}|\F_t\big]\big]\\
 	&=0\\
 \E\big[ (II)\cdot 1_{E_{t-1}} \big] &= \E\big[1_{E_{t-1}} \big|\nabla f(\theta_t) \cdot \sqrt{f(\theta_t)}\,Y_{\theta_t,\xi_t}\big|^2\big]\\
  	&\leq \E\big[\big|\nabla f(\theta_t)\big|^2\,1_{E_{t-1}}\big|\sqrt{f(\theta_t)}\,Y_{\theta_t,\xi_t}\big|^2\big]\\
 	&\leq \E\big[2C_Lf(\theta_t)\,1_{E_{t-1}}\big|\sqrt{f(\theta_t)}\,Y_{\theta_t,\xi_t}\big|^2\big]\\
	&\leq 2C_L\eps\,\E\big[ 1_{E_{t-1}}\,\E\big[ \big|\sqrt{f(\theta_t)}\,Y_{\theta_t,\xi_t}\big|^2 | \F_t\big]\big]\\
	&\leq 2C_L\eps\,\E\big[ 1_{E_{t-1}} f(\theta_t)\,\E\big[ \big|Y_{\theta_t,\xi_t}\big|^2 | \F_t\big]\big]\\
	&= 2C_L \eps\,\E\big[ 1_{E_{t-1}}\,f(\theta_t)\big]\\
\E\big[ (III)\cdot 1_{E_{t-1}} \big] &\leq  \E\big[ 1_{E_{t-1}}\,f(\theta_t)\big]
\end{align*}
where the analysis of $(III)$ reduces to that of $(II)$ and the bound $|\nabla f(\theta)|^2 \leq 2C_L\,f(\theta_t)$ from Lemma \ref{lemma bounded gradient} was used. The result now follows by putting all estimates together.
\end{proof}

We now proceed to estimate the probability that the quadratic noise $R_t$ does not remain small by bounding the probability that it exceeds the given threshold in the $t$-th step and summing over $t$.

\begin{lemma}
The estimate
\begin{equation}\label{eq small probability local convergence}
\P(E_t^c) \leq \frac{C\sigma\eta^2 \big(2C_L\eps +1\big)}r \sum_{i=0}^t \E\big[ 1_{E_{i-1}} f(\theta_i)\big]
\end{equation}
holds.
\end{lemma}

\begin{proof}
{\bf First step.} The probably that $E_t$ does not occur coincides with the probability that there exists some $i\leq t$ such that $R_i$ exceeds $r_\delta$ at $i$, but not $i-1$. More precisely 
\[
\bigcup_{i=0}^{t-1} \widetilde E_i = \bigcup_{i=0}^{t-1} (E_{i-1}\setminus E_{i}) = E_{-1}\setminus E_t = E_t^c
\]
since $E_0$ is the whole space.
Hence
\begin{align*}
\P(E_t^c) &= \sum_{i=0}^{t-1} \P(\widetilde E_i).
\end{align*}
Using $\widetilde E_i = E_{i-1} \cap \{R_i>r\}$ and $1_{\widetilde E_i} = 1_{E_{i-1}} 1_{\{R_i>r\}}$, we bound
\begin{equation}\label{eq local conv first estimate}
\P(\widetilde E_i) = \E\big[1_{E_{i-1}} 1_{\{R_i>r\}}\big]
	\leq \E\left[ 1_{E_{i-1}}\,\frac{R_i}r\right]
	\leq \frac{\E\big[R_i\,1_{E_{i-1}}\big]}{r}
\end{equation}
since $R_i\geq 0$ as a sum of squares.

{\bf Second step.} From \eqref{eq weird estimate}, we obtain
\[
\E\big[ R_t\,1_{E_{t-1}}\big] - \E\big[R_{t-1}\,1_{E_{t-2}}\big] \leq C\sigma\eta^2 \big(2C_L\eps +1\big) \E\big[ 1_{E_{t-1}} f(\theta_t)\big] - r \,\P\big(\widetilde E_t\big),
\]
so by the telescoping sum identity
\begin{equation}\label{eq local conv second estimate}
\E\big[R_t\,1_{E_{t-1}}\big] - \E\big[R_{-1}\,1_{E_{-2}}\big] \leq C\sigma\eta^2 \big(2C_L\eps +1\big)  \sum_{i=0}^t \big\{\E\big[ 1_{E_{i-1}} f(\theta_i)\big] - r \,\P\big(\widetilde E_i\big)\big\}.
\end{equation}
Note that the second term on the right hand side vanishes since the sum defining $R_{-1}$ is empty.

{\bf Conclusion.} Combining \eqref{eq local conv first estimate} and \eqref{eq local conv first estimate}, we find that
\[
\P(\widetilde E_t) \leq \frac{\E\big[R_t\,1_{E_{t-1}}\big]}{r} \leq \frac{C\sigma\eta^2 \big(2C_L\eps +1\big)}r \sum_{i=0}^t \E\big[ 1_{E_{i-1}} f(\theta_i)\big] - \sum_{i=0}^{t}\P\big(\widetilde E_i\big),
\]
so
\begin{align*}
\P(E_t^c) &= \sum_{i=1}^{t-1} \P(\widetilde E_i)
	\leq \frac{C\sigma\eta^2 \big(2C_L\eps +1\big)}r \sum_{i=0}^t \E\big[ 1_{E_{i-1}} f(\theta_i)\big].
\end{align*}
\end{proof}

Finally, we are in a position to prove the local convergence result.

\begin{proof}[Proof of Theorem \ref{theorem local convergence}]
{\bf Step 1.} Since $E_{i}\subseteq E_{i-1}$ and $\Omega_i\subseteq E_{i-1}$, we have
\[
\E\big[ 1_{E_{i-1}} f(\theta_i)\big] \leq \E\big[ 1_{E_{i-1}} \E\big[f(\theta_i) | \F_{i-1}\big]\big]\leq \rho_\eta\,\E\big[f(\theta_{i-1}1_{E_{i-1}}\big] \leq \rho_\eta \,\E\big[ f(\theta_{i-1})\,1_{E_{i-2}}\big]
\]
where 
\[
\rho_\eta = 1 - \Lambda\eta + \eta^2 \frac{C_L(\Lambda+\sigma)}2<1
\]
as previously for functions which satisfy the \L{}ojasiewicz inequality globally. We conclude from \eqref{eq local conv second estimate} that
\[
\P(E_t^c) \leq \frac{C\sigma\eta^2 \big(2C_L\eps +1\big)}r\,\E\big[f(\theta_0)\big] \sum_{i=0}^t \rho_\eta^i \leq \frac{C\sigma\eta^2 \big(2C_L\eps +1\big)}{r(1-\rho_\eta)}\,\eps'.
\]
Thus for every $\delta>0$, there exists $\eps'>0$ such that $\P(E_t) \geq 1-\delta$ for all $t\in\N$ if $f(\theta_0) <\eps'$ almost surely.

{\bf Step 2.} Consider the event 
\[
\widetilde \Omega:= \bigcap_{t\geq 0} \Omega_t(\eps)\quad\text{which satisfies}\quad \P(\widetilde \Omega) \geq 1-\delta
\]
since $E_{t-1}\subseteq \Omega_t$ and $E_t\subseteq E_{t-1}$. Then
\[
\E\big[ f(\theta_t) \,1_{\widetilde\Omega}\big] \leq \E\big[ f(\theta_t) \,1_{\Omega_{t-1}}\big] \leq \rho_\eta^t\,\E\big[f(\theta_0)\big]
\]
as in Step 1. We conclude that for every $\beta\in [1,\rho_\eta^{-1})$ the estimate
\[
\limsup_{t\to\infty} \beta^tf(\theta_t) =0
\]
holds almost surely conditioned on $\widetilde\Omega$ as in the proof of Theorem \ref{theorem lojasiewicz}.
\end{proof}

\section{Proof of Theorem \ref{theorem global convergence}: Global Convergence}\label{appendix global convergence}

Again, we split the proof up over several Lemmas. First, we show that 
\[
\P\big(\liminf_{t\to\infty} f(\theta_t) \leq S\big) = 1.
\]
We always assume that $f$ satisfies the conditions of Theorem \ref{theorem global convergence}, although weaker conditions suffice in the individual steps.

\begin{lemma}\label{lemma expectation remains finite}
If $\E\big[f(\theta_0)\big]<\infty$, the estimate
\[
\sup_{t\geq 0} \E\big[f(\theta_t)\big]<\infty.
\]
holds.
\end{lemma}

\begin{proof}
Recall that for any $\theta$ we have
\[
\E_\xi\big[f(\theta - \eta g(\theta,\xi))\big] \leq \left(1+ \frac{C_L\sigma\eta^2}2\right)f(\theta) -\left( 1- \frac{C_L\eta}2\right)\eta \,\E\big[|\nabla f(\theta)|^2\big]
\]
as in the proof of Theorem \ref{theorem lojasiewicz}. We distinguish two cases:
\begin{itemize}
\item If $f(\theta)\leq S$, then 
\[
\E_\xi\big[f(\theta - \eta g(\theta,\xi))\big] \leq \left(1+ \frac{C_L\sigma\eta^2}2\right)f(\theta) \leq \left(1+ \frac{C_L\sigma\eta^2}2\right)S.
\]
\item If $f(\theta)\geq S$, then
\[
\E_\xi\big[f(\theta - \eta g(\theta,\xi))\big] \leq \left(1- \Lambda \eta + \frac{C_L(\Lambda+\sigma)}{2\Lambda}\eta^2\right) f(\theta)
\]
due to the \L{}ojasiewicz inequality on the set where $f$ is large.
\end{itemize}

In particular, since $f$ is non-negative, we have
\begin{align*}
\E\big[f(\theta_{t+1})\big] &\leq \left(1+ \frac{C_L\sigma\eta^2}2\right)\E\big[ f(\theta_t) \,1_{\{f(\theta_t)\leq S\}}\big] + \tilde\rho_\eta\E\big[f(\theta_t)\,1_{\{f(\theta_t)>S\}}\big]\\
	&\leq \left(1+ \frac{C_L\sigma\eta^2}2\right)S + \tilde\rho_\eta \E\big[f(\theta_{t})\big] 
\end{align*}
where $\tilde \rho_\eta = 1- \Lambda \eta + \frac{C_L(\Lambda+\sigma)}{2\Lambda}\eta^2$. If we abbreviate $z_t = \E\big[f(\theta_t)\big]$, we deduce that 
\[
z_{t+1} \leq \left(1+ \frac{C_L\sigma\eta^2}2\right)S + \tilde\rho_\eta\,z_t \leq \max\left\{\frac{1+\tilde\rho_\eta}2\,z_t, \:\frac{\left(1+ \frac{C_L\sigma\eta^2}2\right)S}{1 - \frac{1+\tilde\rho_\eta}2}\right\}.
\]
Thus if $z_t$ is large, then $z_t$ decays. In fact
\[
\limsup_{t\to\infty} z_t \leq \frac{\left(1+ \frac{C_L\sigma\eta^2}2\right)S}{1 - \rho_\eta} = \frac{\left(1+ \frac{C_L\sigma\eta^2}2\right)S}{\Lambda\eta - \frac{C_L(\Lambda+\sigma)}{2\Lambda}\eta^2}
\]
independently of the initial condition.
\end{proof}

Note that the finiteness of the bound hinges on the fact that the learning rate remains uniformly positive in this simple proof. In other variants of gradient flow, it can be non-trivial to control the possibility of escape. 

The trajectories of SGD satisfy stronger bounds than the expectations. 

\begin{lemma}\label{lemma liminf S}
\[
\P\big(\liminf_{t\to\infty} f(\theta_t) \leq S\big) = 1.
\]
\end{lemma}

\begin{proof}
Let 
\begin{align*}
\Omega_{n, N} &= \{f(\theta_t)\geq S \text{ for all } n< t\leq N\}, \qquad
\Omega_n = \bigcap_{N\geq n}\Omega_{n,N}
\end{align*}
In particular, $\Omega_n \subseteq \Omega_{n,N}$ for all $N\geq n$ and $1_{\Omega_n}\leq 1_{\Omega_{n,t+1}}\leq 1_{\Omega_{n,t}}$ for all $t\geq n$. Thus
\begin{align*}
\P(\Omega_n) &= \E\big[1_{\Omega_n}\big]\\
	&\leq \frac1S\,\E\big[1_{\Omega_{n}}f(\theta_{t+1})\big]\\
	&\leq \frac1S\,\E\big[1_{\Omega_{n,t}}f(\theta_{t+1})\big]\\
	&= \frac1S\,\E\big[ \E\big[1_{\Omega_{n,t}}\,f(\theta_{t+1}) |\F_t\big]\big]\\
	&\leq  \frac{\rho_\eta}S\E\big[ 1_{\Omega_{n,t}} f(\theta_t)\big]\\
	&\leq \dots\\
	&\leq \rho_\eta^{t-n} \frac{\E\big[ f(\theta_n)\big]}S\\
	&\leq \rho_\eta^{t-n} \frac{\sup_{s\in\N}\E\big[ f(\theta_s)\big]}S
\end{align*}
for any $t\geq n$. Thus $\P(\Omega_n) = 0$ for all $n\in\N$. Hence also
\[
\P\big(\liminf_{t\to\infty} f(\theta_t) > S\big) \leq \P\left(\bigcup_{n=1}^\infty\Omega_n \right) =0.
\]
\end{proof}

We have shown that we visit the set $\{f<S\}$ infinitely often almost surely. We now show that in every visit $\theta_t$, the probability that $f(\theta_{t+1}) < \eps'$ is uniformly positive. Below, we will use this to show that we visit the set $\{f<\eps'\}$ infinitely often with uniformly positive probability (which then implies that SGD iterates approach the set of minimizers almost surely). In this step, we use that the noise is uniformly `spread out'.

\begin{lemma}
There exists $\gamma>0$ such that the following holds: If $f(\bar\theta)\leq S$, then 
\[
\P \big( f(\bar\theta- \eta g(\bar\theta,\xi)\big) < \eps'\big) \geq \gamma.
\]
\end{lemma}

\begin{proof}
We consider two cases separately: $f(\theta)<\eps'$ or $f(\theta)>\eps'$. In the first case, we can argue by considering the gradient descent structure, while we rely on the stochastic noise in the second case.

{\bf First case.} If $f(\bar\theta)<\eps'$, then
\[
\E\big[ f(\bar\theta- \eta g(\bar\theta,\xi)\big)\big] \leq \rho_\eta\,f(\bar\theta)
\]
since $f$ satisfies a \L{}ojasiewicz inequality in this region. In particular
\[
\P\big(f(\bar\theta- \eta g(\bar\theta,\xi)\big)\geq f(\bar\theta)\big)\leq \frac{\E\big[ f(\bar\theta- \eta g(\bar\theta,\xi)\big)\big]}{f(\bar\theta)} \leq \rho_\eta<1.
\]

{\bf Second case.}
By assumption, the set of moderate energy is not too far from the set of global minimizers in Hausdorff distance, i.e.\ there exists $\theta'$ such that
\begin{enumerate}
\item $f(\theta') =0$ and
\item $|\bar\theta-\theta'|<R$.
\end{enumerate}
Due to the Lipschitz-continuity of the gradient of $f$, there exists $\tilde r>0$ depending only on $C_L$ such that $f(\theta)<\eps'$ for all $\theta\in B_{\tilde r}(\theta')$. We conclude that
\begin{align*}
\P\big( f(\bar\theta- \eta g(\bar\theta,\xi)\big) < r\big) &\geq \P\big(\bar\theta- \eta g(\bar\theta,\xi) \in B_{\tilde r}(\theta')\big)\\
	&= \P\left(Y_{\theta,\xi} \in B_{\frac{\tilde r}{\eta\sqrt{f(\theta)}}}\left(\frac{\bar\theta-\theta' - \eta\,\nabla f(\bar\theta)}{\eta\sqrt{f(\theta)}}\right)\right).	
\end{align*}
By assumption, the radius
\[
\frac{\tilde r}{\eta\sqrt{f(\theta)}} \geq \frac{\tilde r}{\eta\sqrt{S}}  >0
\]
is uniformly positive and the center of the ball
\[
\left| \frac{\theta'-\bar\theta - \eta\,\nabla f(\bar\theta)}{\eta\sqrt{f(\theta)}}\right| \leq \frac{|\theta'-\bar\theta| + \eta \sqrt{2C_LS}}{\eta\,\sqrt {\eps'}} \leq \frac{R + \eta \sqrt{2C_LS}}{\eta\,\sqrt{\eps'}}
\]
is in some large ball independent of $\theta$. Thus the probability of jumping into $\{f< \eps'\}$, albeit small, is uniformly positive with a lower bound
\[
\P\big( f(\bar\theta- \eta g(\bar\theta,\xi)\big)\in B_{r_\delta}(\theta')\big) \geq \inf\left\{\psi(s,r) : s\leq \frac{R + \eta \sqrt{2C_LS}}{\eta\,\sqrt{\eps'}}, \: r > \frac{\tilde r}{\eta\sqrt{S}} \right\}>0.
\]
\end{proof}

By Lemma \ref{lemma liminf S}, the sequence of stopping times $\tau_0=0$, 
\[
\tau_k = \inf\{n> \tau_k+1 : f(\theta_{n})\leq S\}
\]
is well-defined except on a set of measure zero. Consider the Markov process
\[
Z_{2k} = \theta_{\tau_k}, \qquad Z_{2k+1} = \theta_{\tau_k+1}.
\]
Note that we use $\tau_k+1$ for odd times, not $\tau_{k+1}$. To show the stronger statement that
\[
\P\big(\liminf_{t\to\infty} f(\theta_t) \leq \eps'\big) = 1,
\]
we use the conditional Borel-Cantelli Lemma \ref{lemma conditional Borel-Cantelli}.

\begin{lemma}\cite[\"Ubung 11.2.6]{klenke2006wahrscheinlichkeitstheorie}\label{lemma conditional Borel-Cantelli}
Let $\F_n$ be a filtration of a probability space and $A_n$ a sequence of events such that $A_n\in \F_n$ for all $n\in \N$. Define
\[
A^* = \left\{\sum_{n=1}^\infty \E\big[1_{A_n} | \F_{n-1}\big] = \infty\right\}, \qquad A_\infty = \limsup_{n\to \infty} A_n = \big\{A_n \text{ for infinitely many }n\in\N\big\}.
\]
Then $\P(A^*\Delta A_\infty) = 0$ where $A\Delta B$ denotes the symmetric difference of $A$ and $B$.
\end{lemma}

\begin{corollary}\label{corollary liminf global}
\[
\P\big(\liminf_{t\to\infty} f(\theta_t) \leq \eps'\big) = 1.
\]
\end{corollary}

\begin{proof}
Consider the filtration $\F_n$ generated by $Z_n$ and the events
\[
A_n = \{f(\theta_n) < \eps'\}.
\]
Then
\begin{align*}
\sum_{n=1}^\infty \E\big[1_{A_n} | \F_{n-1}\big] &\geq \sum_{n=1}^\infty \E\big[1_{A_{2n+1}} | \F_{2n}\big]
	\geq \sum_{n=1}^\infty \gamma
	= +\infty
\end{align*}
except on the null set where $\tau_k$ is undefined for some $k$.
Thus $\P\big(\limsup_{n\to\infty}A_n\big) = 1$, so almost surely there exist infinitely many $k\in\N$ such that $f(Z_k) < \eps'$. In particular, almost surely there exist infinitely many $t\in \N$ such that $f(\theta_t)<\eps'$.
\end{proof}

We are now ready to prove the global convergence result. 

\begin{proof}[Proof of Theorem \ref{theorem global convergence}]
Let $\beta\in [1,\rho_\eta^{-1})$ and consider the event
\[
\widehat \Omega_\beta:= \left\{ \limsup_{t\to\infty} \beta^tf(\theta_t) = 0\right\}.
\]
Choose $\delta\in(0,1)$ and associated $\eps'>0$. By Corollary \ref{corollary liminf global}, the stopping time
\[
\tau:= \inf\{t\geq 0 : f(\theta_t) < \eps'\}
\]
is finite except on a set of measure zero. Consider $\theta_\tau$ as the initial condition of a different SGD realization $\tilde\theta$ and note that the conditional independence properties which we used to obtain decay estimates still hold for the gradient estimators $\widetilde g_t = g(\theta_{t+\tau}, \xi_{t+\tau})$ with respect to the $\sigma$-algebras $\widetilde F_t$ generated by the random variables $\theta_{t+\tau}$ for $t\geq 0$. 

By Theorem \ref{theorem local convergence}, we observe that with probability at least $1-\delta$, we have
\[
\limsup_{t\to\infty} \beta^tf(\theta_{t+\tau}) =0.
\]
Thus for any $T>0$, with probability at least $1-\delta - \P(\tau>T)$ we have 
\[
\limsup_{t\to\infty} \beta^tf(\theta_{t}) \leq \limsup_{t\to\infty} \beta^{t+\tau}f(\theta_{t+\tau}) \leq \beta^T\cdot 0 = 0.
\]
Taking $T\to \infty$ and $\delta\to 0$, we find that almost surely
\[
\limsup_{t\to\infty} \beta^tf(\theta_{t}) =0.
\]
\end{proof}

We conclude by proving that not just the function values $f(\theta_t)$, but also the arguments $\theta_t$ converge.

\begin{proof}[Proof of Corollary \ref{corollary global convergence argument}]
The proof follows the same lines as that of Corollary \ref{corollary convergence to minimum lojasiewicz}. Consider the set
\[
U_T = \big\{f(\theta_t)\leq \beta^t \text{ for all }t\geq T\}.
\]
Evidently
\begin{align*}
\E\big[ \big|\theta_{t+1}-\theta_t\big|^2 1_{U_T}\big] &= \eta^2\E\big[ 1_{U_T} |g(\theta_t,\xi_t)|^2\big]\\
	&\leq \eta^2 \E\big[ 1_{U_T} |\nabla f(\theta_t)|^2\big]+ \E\big[1_{U_T}f(\theta_t)\big]\\
	&\leq \eta^2\left(2C_L+\sigma\right) \beta^t.
\end{align*}
As in the proof of Corollary \ref{corollary convergence to minimum lojasiewicz}, we deduce that $\theta_t$ converges pointwise almost everywhere on the set $U_T$. As this is true for all $T$ and 
\[
\P\left(\bigcup_{T=1}^\infty U_T \right) = 1.
\]
due to Theorem \ref{theorem global convergence}, we find that $\theta_t$ converges almost surely to a random variable $\theta_\infty$ and $f(\theta_\infty) = \lim_{t\to\infty} f(\theta_t) =0$.
\end{proof}

\end{document}